\documentclass[conference]{IEEEtran} 
\IEEEoverridecommandlockouts
\usepackage{enumerate}
\usepackage{subfigure}
\usepackage{algorithm,amsthm}
\usepackage{algpseudocode}
\usepackage{cite}
\usepackage{amsmath,amssymb,amsfonts}   
\usepackage{graphicx}
\usepackage{textcomp}
\usepackage{xcolor}
\usepackage{authblk}
\usepackage{float}
\usepackage[T1]{fontenc}
\usepackage[utf8]{inputenc}
\usepackage{array}
\usepackage{pifont}
\usepackage[justification=centering]{caption}
\usepackage{stfloats}
\usepackage{verbatim}
\usepackage{soul}
\usepackage{tikz}
\usepackage{etoolbox}
\usepackage{enumitem}
\usepackage{makecell}
\usepackage{multirow}
\usepackage{booktabs}
\newtheorem{definition}{\bf Definition}
\newtheorem{theorem}{\bf Theorem}
\newtheorem{lemma}{\bf Lemma}
\newtheorem{example}{\bf Example}
\newtheorem{corollary}{\bf Corollary}
\newcommand*{\circled}[1]{\lower.7ex\hbox{\tikz\draw (0pt, 0pt)%
		circle (.5em) node {\makebox[1em][c]{\small #1}};}}
\robustify{\circled}
\def\BibTeX{{\rm B\kern-.05em{\sc i\kern-.025em b}\kern-.08em
    T\kern-.1667em\lower.7ex\hbox{E}\kern-.125emX}}
    
\begin{document}

	\title{Analyzing and Optimizing Perturbation of DP-SGD Geometrically}

\author{
	\vspace{0.08in}
	Jiawei Duan$^{\ast}$, Haibo Hu $^\ast$, Qingqing Ye $^\ast$, Xinyue Sun $^\dagger$\\
	$^\ast$ The Hong Kong Polytechnic University \\	
	$^\dagger$ Harbin Institute of Technology\\
	\emph{ jiawei.duan@connect.polyu.hk; haibo.hu@polyu.edu.hk; qqing.ye@polyu.edu.hk;} \emph{xysun@hit.edu.cn}
	\vspace{-0.2in}
}

\maketitle
	\begin{abstract}
		Differential privacy (DP) has become a prevalent privacy model in a wide range of machine learning tasks, especially after the debut of DP-SGD. However, DP-SGD, which directly perturbs gradients in the training iterations, fails to mitigate the negative impacts of noise on gradient direction. As a result, DP-SGD is often inefficient. Although various solutions (e.g., clipping to reduce the sensitivity of gradients and amplifying privacy bounds to save privacy budgets) are proposed to trade privacy for model efficiency, the root cause of its inefficiency is yet unveiled. 
		
		In this work, we first generalize DP-SGD and theoretically derive the impact of DP noise on the training process. Our analysis reveals that, in terms of a perturbed gradient, only the noise on direction has eminent impact on the model efficiency while that on magnitude can be mitigated by optimization techniques, i.e., fine-tuning gradient clipping and learning rate. Besides, we confirm that traditional DP introduces biased noise on the direction when adding unbiased noise to the gradient itself. Overall, the perturbation of DP-SGD is actually sub-optimal from a geometric perspective. Motivated by this, we design a geometric perturbation strategy GeoDP within the DP framework, which perturbs the direction and the magnitude of a gradient, respectively. By directly reducing the noise on the direction, GeoDP mitigates the negative impact of DP noise on model efficiency with the same DP guarantee. Extensive experiments on two public datasets (i.e., MNIST and CIFAR-10), one synthetic dataset and three prevalent models (i.e., Logistic Regression, CNN and ResNet) confirm the effectiveness and generality of our strategy.
	\end{abstract}
		
	\begin{IEEEkeywords}
		local differential privacy; federated learning; convergence analysis; optimization strategy
	\end{IEEEkeywords}
	\section{Introduction}
\label{sec:intro}
Although deep learning models have numerous applications in various domains, such as personal recommendation and healthcare, the privacy leakage of training data from these models has become a growing concern. There are already mature attacks which successfully reveal the contents of private data from deep learning models~\cite{carlini2021extracting,gong2021inversenet}. For example, a white-box membership inference attack can infer whether a single data point belongs to the training dataset of a DenseNet with 82\% test accuracy~\cite{8835245}. These attacks pose imminent threats to the wider adoption of deep learning in business sectors with sensitive data, such as healthcare and fintech.

To address this concern, differential privacy (DP), which can provide quantitative amount of privacy preservation to individuals in the training dataset, is embraced by the most prevalent optimization technique of model training, i.e., stochastic gradient descent (SGD). Referred to as DP-SGD~\cite{li2021federated,zeighami2022neural,liu2021projected,bao2022skellam}, this algorithm adds random DP noise to gradients in the training process so that attackers cannot infer private data from model parameters with a high probability.

However, a primary drawback of DP-SGD is the ineffective training process caused by the overwhelming noise, which extremely deteriorates the model efficiency. Although much attention~\cite{abadi2016deep,mironov2017renyi,fu2023dpsur} has been paid on reducing the noise scale, the majority of existing solutions, which numerically add DP noise to gradients, do not exploit the geometric nature of SGD (i.e., descending gradient to locate the optima). As reviewed in Section \ref{subsec:DPSGD}, SGD exhibits a distinctive geometric property --- the direction of a gradient rather than the magnitude determines the descent trend. By contrast, regular DP algorithms, such as the Gaussian mechanism~\cite{dwork2014algorithmic}, was originally designed to preserve numerical (scalar) values rather than vector values. As such, there is a distinct gap between directional SGD and numerical DP perturbation, causing at least two limitations in DP-SGD. First, \textbf{existing optimization techniques of SGD (i.e., fine-tuning clipping and learning rate)}, which can effectively reduce the noise on the magnitude of a gradient, \textbf{cannot alleviate the negative impact on the direction}, as illustrated by Example \ref{exam:geometric}. Second, \textbf{traditional DP introduces biased noise on the direction of a gradient}, even if the total noise to the gradient is unbiased (proved in Lemma \ref{lem:MSE}). As a result, the perturbation of traditional DP-SGD is only sub-optimal from a geometric perspective.
\begin{example}
	Suppose that we have a two-dimensional gradient $\boldsymbol{g}=(1,\sqrt{3})$ with its direction $\boldsymbol{\theta}=\arctan(\sqrt{3}/{1})=\pi/3$ and magnitude $\Vert \boldsymbol{g}\Vert=\sqrt{1+3}=2$. Given clipping threshold $C_1=2$, we add noise $\boldsymbol{n}_1=(0.3,0.15)$ to the clipped gradient $\tilde{\boldsymbol{g}}_1=\boldsymbol{g}/\max\left\{1,\Vert \boldsymbol{g}\Vert/C_1\right\}=(1,\sqrt{3})$ and derive the perturbed direction $\boldsymbol{\theta}^*_1=\arctan\frac{\sqrt{3}+0.15}{1+0.3}\approx 0.97$. If $C_2=1$, the clipped gradient and the noise would be $\tilde{\boldsymbol{g}}_2=\boldsymbol{g}/\max\left\{1,\Vert \boldsymbol{g}\Vert/C_2\right\}=(\frac{1}{2},\frac{\sqrt{3}}{2})$ and $\boldsymbol{n}_2=\boldsymbol{n}_1/(C_1/C_2)=(0.15,0.075)$, respectively, as per DP-SGD~\cite{abadi2016deep}. Still, the perturbed direction is $\boldsymbol{\theta}^*_2=\arctan\frac{\frac{\sqrt{3}}{2}+0.075}{\frac{1}{2}+0.15}\approx 0.97$. Although the noise scale is successfully reduced by gradient clipping ($\Vert\boldsymbol{n}_2\Vert<\Vert\boldsymbol{n}_1\Vert$), the perturbation on the direction of a gradient remains the same ($\boldsymbol{\theta}^*_2=\boldsymbol{\theta}^*_1$).
\label{exam:geometric}
\end{example}

\begin{figure}[htbp]
	\centering
	\includegraphics[width=0.7\linewidth]{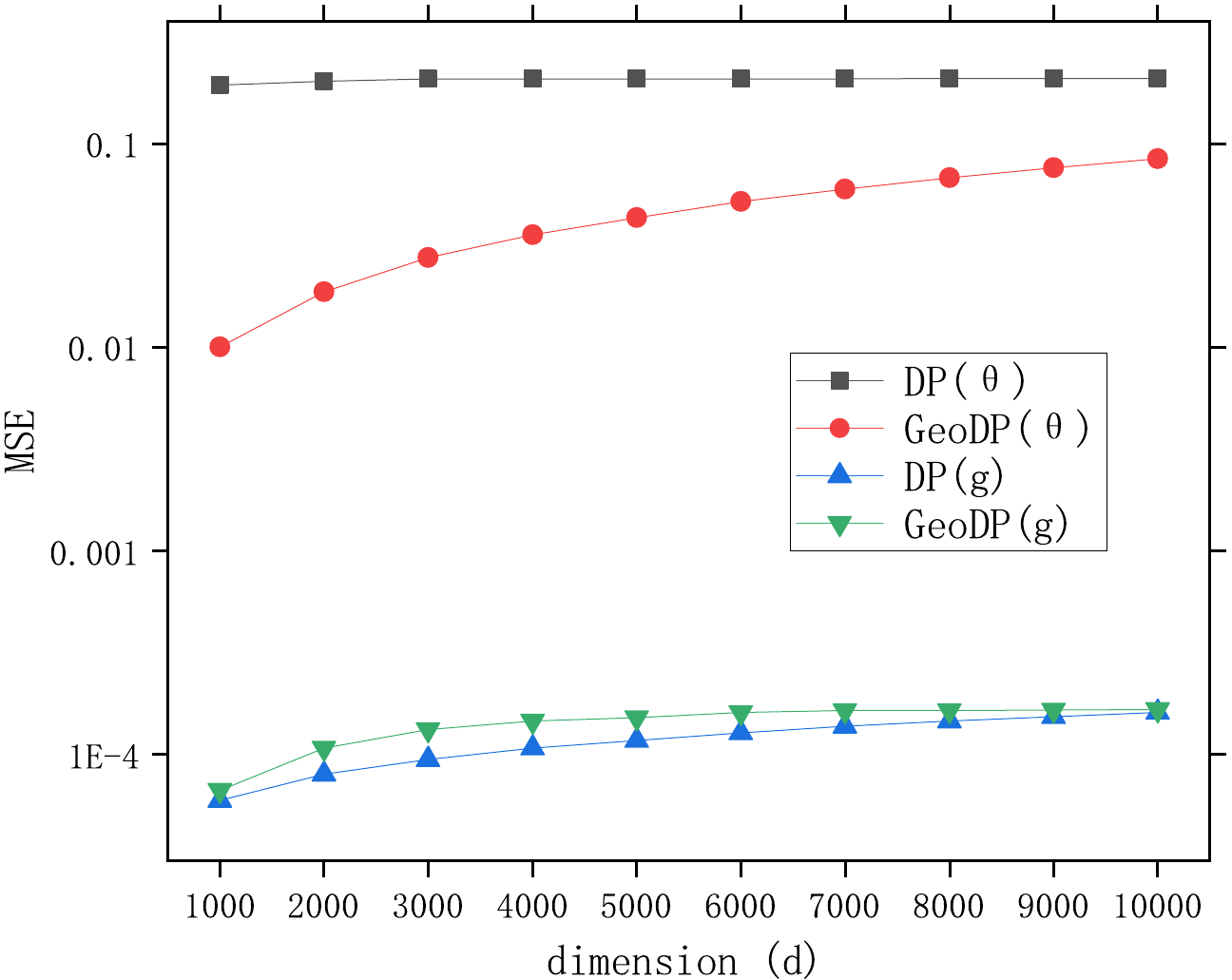}
	\caption{Comparing MSEs of GeoDP and DP on preserving directions and values of gradients under synthetic dataset (composed of gradients from CNN training, as introduced in Section \ref{subsec:setup}). While $\theta$ and $g$ label the MSE of perturbed directions and gradients themselves, experimental results confirm that GeoDP achieves smaller MSEs on perturbed directions (i.e., the red line is below the black one), while sacrificing the accuracy of perturbed gradients (i.e., the green line is above the blue one). In general, GeoDP better preserves directions of gradients while traditional DP only excels in preserving numerical values of gradients.}
	\label{fig:sumerized}
\end{figure}

In this paper, we propose a geometric perturbation strategy GeoDP to address these limitations. First, we theoretically derive the impact of DP noise on the efficiency of DP-SGD. Proved by this fine-grained analysis, the perturbation of DP-SGD, which introduces biased noise to the direction of a gradient, is actually sub-optimal. Inspired by this, we propose a geometric perturbation strategy \textit{GeoDP} which perturbs both the direction and the magnitude of a gradient, so as to relieve the noisy gradient direction and optimize model efficiency with the same DP guarantee. Figure \ref{fig:sumerized} illustrates empirical performances of GeoDP and DP to support the superiority of GeoDP in the perspective of geometry. Such experimental results can also be confirmed in our theoretical analysis. In summary, our main contributions are as follows:
\begin{itemize}
	\item {To the best of our knowledge, we are the first to prove that the perturbation of traditional DP-SGD is actually sub-optimal from a geometric perspective.}
	\item {Within the classic DP framework, we propose a geometric perturbation strategy \textit{GeoDP} to directly add the noise on the direction of a gradient, which rigorously guarantees a better trade-off between privacy and efficiency.}
	\item {Extensive experiments on public datasets as well as prevalent AI models validate the generality and effectiveness of GeoDP.}
\end{itemize}

The rest of this paper is organized as follows. Section \ref{sec:literature} reviews the related literature. Section \ref{sec:problem} introduces basic concepts as well as formulating problems. Section \ref{sec:analysis} presents our theoretical analysis on deficiency of DP-SGD while Section \ref{sec:strategy} presents the perturbation strategy \textit{GeoDP}. Experimental results are in Section \ref{sec:experiment}, followed by a conclusion in Section \ref{sec:conclusion}.

	\section{Literature review}
\label{sec:literature}
In this section, we review related works from three aspects: DP, SGD and their crossover works DP-SGD.
%and mainstream attacks to defend. 
\subsection{Differential Privacy (DP)}
DP~\cite{dwork2014algorithmic,wasserman2010statistical} is a framework designed to provide strong privacy guarantees for datasets whose data is used in data analysis or machine learning models. It aims to allow any third party, e.g., data scientists and researchers, to glean useful insights from datasets while ensuring that the privacy of individuals cannot be compromised. The core idea of differential privacy is that a query to a database should yield approximately the same result whether any individual person's data is included in the database or not. This is achieved by adding noise to the data or the query results, which helps to obscure the contributions of individual data points. 

Since Dwork \emph{et al.}~\cite{10.1007/11681878_14} first introduced the definition of \emph{differential privacy} (DP), DP has been extended to various scopes, such as numerical data collection~\cite{duchi2018minimax,wang2019collecting}, set-value data collection~\cite{chen2011publishing,wang2020set}, key-value data collection~\cite{ye2021privkvm}, high-dimensional data~\cite{duan2022utility}, graph analysis~\cite{sun2023puts}, time series data release~\cite{ye2021beyond}, private learning~\cite{zheng2019bdpl,fu2023dpsur}, federated matrix factorization~\cite{li2021federated}, data mining~\cite{hu2015differential}, local differential privacy~\cite{xu2020collecting,xu2019dpsaas,bao2021cgm,wang2018privset}, database query~\cite{farias2023local,bogatov2021varepsilonpsolute}, markov model~\cite{xiao2017loclok} and benchmark~\cite{schaler2023benchmarking,duan2022utility,duan2024ldptube}. Relevant to our work, we follow the common practice to implement Gaussian mechanism~\cite{dwork2014algorithmic} to perturb model parameters. Besides, Rényi Differential Privacy (RDP)~\cite{mironov2017renyi} allows us to more accurately estimate the cumulative privacy loss of the whole training process. 

\subsection{Stochastic Gradient Descent (SGD)}
Stochastic Gradient Descent (SGD) is a fundamental optimization algorithm widely used in machine learning and deep learning for training a wide array of models. It is especially popular for its efficiency in dealing with large datasets and high-dimensional optimization problems. SGD was first introduced by Herbert \emph{et al.}~\cite{robbins1951stochastic}, and applied for training deep learning models~\cite{rumelhart1986learning}. The development of SGD has seen several significant improvements over the years. Xavier \emph{et al.}~\cite{glorot2010understanding} and Yoshua~\cite{bengio2012practical} optimized deep neural networks using SGD. Momentum, a critical concept to accelerate SGD, was emphasized by Llya \emph{et al.}~\cite{sutskever2013importance}. Diederik \emph{et al.}~\cite{kingma2014adam} proposed Adam, a variant of SGD that adaptively adjusts the learning rate for each parameter. Sergey \emph{et al.}~\cite{ioffe2015batch} introduced Batch Normalization, a technique to reduce the internal covariate shift in deep networks. Yang \emph{et al.}~\cite{you2017large} and Zhang \emph{et al.}~\cite{zhang2019lookahead} further proposed large-batch training and lookahead optimizer, respectively. These advancements have pushed the boundaries of SGD, enabling efficient training of increasingly complex deep learning models~\cite{xu2024stochastic,zhang2024metastore,wang2024optimizing,xing2024database}. Without loss of generality, we follow the common practice of existing works and implement SGD without momentum to better demonstrate the efficiency of our strategy. 

\subsection{Differentially Private Stochastic Gradient Descent (DP-SGD)}
\label{subsec:DPSGD}
As a privacy-preserving technique for training various models, DP-SGD is an adaptation of the traditional SGD algorithm to incorporate differential privacy guarantees. This is crucial in applications where data confidentiality and user privacy are concerns, such as in medical or financial data processing. The basic idea is adding DP noise to gradients during the training process. Chaudhuri et al.~\cite{chaudhuri2011differentially} initially introduced a DP-SGD algorithm for empirical risk minimization. Abadi et al.~\cite{abadi2016deep} were one of the first to introduce DP-SGD into deep learning. Afterwards, DP-SGD has been rapidly applied to various models, such as generative adversarial network~\cite{ho2021dp}, Bayesian learning~\cite{NIPS2017_dfce0680}, federated learning~\cite{zhang2022understanding}.

As for optimizing model efficiency of DP-SGD, there are three major streams. First, gradient clipping can help to reduce the noise scale while still following DP framework. For example, adaptive gradient clipping~\cite{xia2023differentially,zhang2022understanding,10.5555/3495724.3496879}, which adaptively bounds the sensitivity of the DP noise, can trade the clipped information for noise reduction. Second, we can amplify the privacy bounds to save privacy budgets, such as Rényi Differential Privacy~\cite{gopi2021numerical}. Last, more efficient SGD algorithms, such as DP-Adam \cite{tang2024dp}, can be introduced to DP-SGD so as to improve the training efficiency. 

However, existing works still cling to numerical perturbation, and there is no work investigating whether the numerical DP scheme is optimal for the geometric SGD in various applications. In this work, we instead fill in this gap \textbf{by proposing a new DP perturbation scheme}, which exclusively preserves directions of gradients so as to improve model efficiency. As no previous works carry out optimization from this perspective, \textbf{our work is therefore only parallel to vanilla DP-SGD while orthogonal to all existing works}. 

	\section{Preliminaries and problem formulation}
\label{sec:problem}
In this section, we first introduce the preliminaries of DP and SGD, based on which we then formulate DP-SGD as an optimization problem.
\begin{table}[t]% h asks to places the floating element [h]ere.
  \centering
  \begin{tabular}{cc}
	    \toprule
	    Symbol & Meaning \\
	    \midrule
	    $\epsilon$&privacy budget\\
	    $\beta$&bounding factor\\
%	    $D$ &database\\
%	    $S$&subset\\
%	    $s$&one training data\\	    
	    $B$&batch size\\
%	    $T$&total number of iterations\\
%	    $t$&current iteration\\
	    $C$&clipping threshold\\
	    $\sigma$ &noise multiplier\\
%	    $\eta$ &learning rate\\
%	    $l$&loss function\\
	    $\boldsymbol{w}$&model parameters\\
	    $\boldsymbol{w}^{\star}$&global optima\\
	    $\boldsymbol{g}$&original gradient\\
	    $\tilde{\boldsymbol{g}}$&clipped gradient\\
	    $\boldsymbol{n}$&DP noise vector\\
	    $\boldsymbol{g}^*$&perturbed gradient from traditional DP\\
	    $\boldsymbol{g}^\star$&perturbed gradient from GeoDP\\
	    $\boldsymbol{\theta}$&direction of a gradient\\
	    $\Vert\boldsymbol{g}\Vert$&magnitude of a gradient\\
	  \bottomrule
 \end{tabular}
\caption{Frequently-used notations}
\label{tab:notation}
\end{table}
\subsection{Differential Privacy}
\label{subsec:DP}
Differential Privacy (DP) is a mathematical framework that quantifies the privacy preservation. Formally, $(\epsilon,\delta)$-\emph{DP} is defined as follows:
\begin{definition}
	\emph{($(\epsilon$,$\delta)$-DP).} A randomized algorithm $\mathcal{M}: D\to R$ satisfies $(\epsilon$,$\delta)$-DP if for all datasets $D$ and $D'$ differing on a single element, and for all subsets $S$ of $R$, the following inequality always holds:
	\begin{equation}
		\Pr[\mathcal{M}(D) \in S] \leq e^\epsilon \times \Pr[\mathcal{M}(D') \in S] + \delta.
	\end{equation}
\end{definition}
In essence, DP guarantees that given any outcome of $\mathcal{M}$, it is unlikely for any third party to infer the original record with high confidence. Privacy budget $\epsilon$ controls the level of preservation. Namely, a lower $\epsilon$ means stricter privacy preservation and thus poorer efficiency, and vice versa. $\delta$ determines the probability of not satisfying $\epsilon$ preservation.

To determine the noise scale for DP, we measure the maximum change of $\mathcal{M}$ in terms of $L_2$-norm as:
\begin{definition}
	\textit{($L_2$-sensitivity).} The $L_2$-sensitivity of $\mathcal{M}$ is:
	\begin{equation}
		\Delta\mathcal{M}=\max_{\Vert D-D'\Vert_1=1} \Vert \mathcal{M}(D)-\mathcal{{M}}(D')\Vert_2.
	\end{equation}
\end{definition}

Through out the paper, we follow the common practice of existing works~\cite{abadi2016deep,fu2023dpsur} and use Gaussian mechanism~\cite{dwork2014algorithmic} for theoretical analysis and experiments. The perturbed value of Gaussian mechanism is $g^*=g+Gau(0,2\Delta\mathcal{M}\ln\frac{1.25}{\delta}/\epsilon^2)$, where $Gau$ denotes a random variable that follows Gaussian distribution with probability density function: 
\begin{equation}
	f(x)=\frac{1}{\sigma\sqrt{2\pi}}\exp(-\frac{(x-\mu)^2}{2\sigma^2}).
\end{equation}
Referring to the standard deviation of $Gau(0,2\ln\frac{1.25}{\delta}/\epsilon^2)$ as \textbf{the noise multiplier $\sigma$}, \textbf{the noise scale} of Gaussian mechanism is $\Delta\mathcal{M}\sigma$~\cite{dwork2014algorithmic}. Thus, a smaller $\sigma$ comes with a lesser perturbation.

%While tradition DP only provides a uniform privacy guarantee for all data, R\'{e}nyi DP (RDP) otherwise allows varying privacy guarantees depending on the data distribution and can therefore provide tighter privacy bounds for DP-SGD.
%\begin{definition}
%	(\emph{R\'{e}nyi DP}). R\'{e}nyi DP (RDP) is a generalization of $(\epsilon,\delta)$-DP that uses R\'{e}nyi divergence as a distance metric.
%	The R\'{e}nyi divergence of order $\alpha$ between two distributions $P$ and $Q$ is defined as:
%	\begin{equation}
%		D_\alpha(P||Q)=\frac{1}{\alpha-1}\log\mathbb{E}_{x\sim P}\left[\left(\frac{P(x)}{Q(x)}\right)^{\alpha-1}\right].
%	\end{equation}
%	A model satisfies $(\alpha,\epsilon)$-RDP if
%	\begin{equation}
%		\begin{aligned}
%			&D_\alpha(\mathcal{M}(D)||\mathcal{M}(D'))\\=&\frac{1}{\alpha-1}\log\mathbb{E}_{x\sim \mathcal{M}(D)}\left[\left(\frac{\Pr\left[\mathcal{M}(D)=x\right]}{\Pr\left[\mathcal{M}(D')=x\right]}\right)^{\alpha-1}\right]\leq \epsilon.			
%		\end{aligned}	
%	\end{equation}
%\end{definition}
%It can be proved~\cite{ren2018textsf} that an $(\alpha, \epsilon)$-RDP guarantee is equivalent to an $\left(\epsilon + \frac{\log(1/\delta)}{\alpha-1}, \delta\right)$-DP guarantee for any $\delta \in \left(0,1\right)$. As $\frac{\log(1/\delta)}{\alpha - 1}>0$ always holds, RDP provides tighter privacy bounds than $(\epsilon,\delta)$-DP.
\subsection{Stochastic Gradient Descent}
SGD (stochastic gradient descent) is one of the most widely used optimization techniques in machine learning~\cite{bottou2012stochastic}. Let $D$ be the private dataset, and $\boldsymbol{w}$ denote the model parameters (a.k.a the training model). Given $S\subseteq D$ and $S=\left\{s_{1}, s_{2}, ..., s_{(B-1)}, s_{B}\right\}$ ($B$ denoting the number of data in $S$), the objective $F(\boldsymbol{w})$ can be formulated as $F(\boldsymbol{w};S)=\frac{1}{B}\sum_{j=1}^{B} \mathit{l} (\boldsymbol{w};s_{j})$,
%\begin{equation}
%	F(\boldsymbol{w};S)=\frac{1}{B}\sum_{j=1}^{B} \mathit{l} (\boldsymbol{w};s_{j}).
%\end{equation}
where $\mathit{l} (\boldsymbol{w};s_{j})$ is the loss function trained on one subset data $s_j$ to optimize $\boldsymbol{w}$.

To optimize this task, we follow the common practice of existing works and use mini-batch stochastic gradient descent (SGD) \cite{lecun2002efficient}. Given the total number of iterations $T$, $\boldsymbol{w}_t=\left(\boldsymbol{w}_{t1}, \boldsymbol{w}_{t2}, ..., \boldsymbol{w}_{t(d-1)}, \boldsymbol{w}_{td}\right)$ $\left(0\leq t\leq T-1\right)$ denotes a $d$-dimensional model weight derived from the $t$-th iteration (where $t=0$ is the initiate state). While using $\eta$ to denote the learning rate, we have the gradient $\boldsymbol{g}_t$ of the $t$-th iteration:
\begin{equation}
	\begin{aligned}
		\boldsymbol{g}_t&=\nabla F(\boldsymbol{w}_t;S)
		=\frac{1}{B}\sum_{j=1}^{B} \nabla\mathit{l} (\boldsymbol{w};s_{j})=\frac{1}{B}\sum_{j=1}^{B}\boldsymbol{g}_{tj}.
	\end{aligned}	
\label{equ:graident}
\end{equation}
%where $\nabla l=\left(\frac{\partial l}{\partial \boldsymbol{w}_{1}}, \frac{\partial l}{\partial \boldsymbol{w}_{2}}, ..., \frac{\partial l}{\partial \boldsymbol{w}_{d-1}}, \frac{\partial l}{\partial \boldsymbol{w}_{d}}\right)$, and respective gradients $\left\{\boldsymbol{g}_{tj}|1\leq j\leq B\right\}$ are derived from respective data $\left\{s_j|1\leq j\leq B\right\}$ of the batch. The $t$-th iteration updates the model weight $\boldsymbol{w}_{t+1}$ as: $\boldsymbol{w}_{t+1}=\boldsymbol{w}_t-\eta \boldsymbol{g}_t$.
\begin{equation}
	\boldsymbol{w}_{t+1}=\boldsymbol{w}_t-\eta \boldsymbol{g}_t
\label{equ:GradientDescent}
\end{equation}

%By tuning the batch size $B$, the analysis on this optimization technique also applies to its variants. For example, if $B=\left|D\right|$, it is equivalent to the batch gradient descent \cite{boyd_vandenberghe_2004}; if $B=1$, it is equivalent to the stochastic gradient descent \cite{bottou2012stochastic}. Throughout this paper, we abbreviate mini-batch stochastic gradient descent and its variants collectively as SGD.

SGD is known to have an intrinsic problem of gradient explosion \cite{pascanu2013difficulty}. It often occurs when the gradients become very large during backpropagation, and causes the model to converge rather slowly. As the most effective solution to this problem, gradient clipping \cite{pascanu2013difficulty} is also considered in this work. Let $\left\|\boldsymbol{g}\right\|$ denote the $L_2$-norm of a $d$-dimensional vector $\boldsymbol{g}=(\boldsymbol{g}_1, \boldsymbol{g}_2, ..., \boldsymbol{g}_{d-1}, \boldsymbol{g}_d)$, i.e., $\left\|\boldsymbol{g}\right\|=\sqrt{\sum_{z=1}^{d}\boldsymbol{g}_z^2}$. Assume that $G$ is the maximum $L_2$-norm value of all possible gradients for any weight $\boldsymbol{w}$ derived from any subset $S$, i.e., $G=\sup_{\boldsymbol{w}\in \mathbb{R}^d, S\in D}\mathbb{E}\left[\left\|\boldsymbol{g}\right\|\right]$. Then each gradient $\boldsymbol{g}$ is clipped by a clipping threshold $C\in \left(0,G\right]$. Formally, the clipped gradient $\tilde{\boldsymbol{g}}$ is:
\begin{equation}
	\tilde{\boldsymbol{g}}=\frac{\boldsymbol{g}}{\max \left\{1,\left\|\boldsymbol{g}\right\|/C\right\}}.
\label{equ:clipping}
\end{equation}

%Another advantage of clipping is to reduce the sensitivity of a gradient, which therefore decreases the noise addition in DP-SGD. The most recent state-of-the-art work proposes AUTO-S~\cite{bu2024automatic} for automatic clipping, which conducts clipping as follows:
%\begin{equation}
%	\tilde{\boldsymbol{g}}=\frac{\boldsymbol{g}}{\left\|\boldsymbol{g}\right\|+0.01}.
%\label{equ:autoclipping}
%\end{equation}

Applying Equation \ref{equ:clipping} to Equation \ref{equ:graident}, we derive the clipped gradient from the $t$-th iteration as: 
%$	 \tilde{\boldsymbol{g}}_t=\frac{1}{B}\sum_{j=1}^{B}\tilde{\boldsymbol{g}}_{tj}$.
\begin{equation}
	\tilde{\boldsymbol{g}}_t=\frac{1}{B}\sum_{j=1}^{B}\tilde{\boldsymbol{g}}_{tj}.
	\label{equ:clipped}
\end{equation}
\subsection{Problem Formulation of DP-SGD}
In each iteration of DP-SGD, $\boldsymbol{w}_{t+1}$ is perturbed to $\boldsymbol{w}_{t+1}^*$ by adding DP noise $\boldsymbol{n}_t$ to the sum of $\tilde{\boldsymbol{g}}_{tj}$. Let $\boldsymbol{g}_{t}^*$ denote the perturbed gradient. Formally,
\begin{equation}
	\begin{aligned}
		\boldsymbol{g}_t^*&=\frac{1}{B}(\sum_{j=1}^{B}\tilde{\boldsymbol{g}}_{tj}+\boldsymbol{n}_t)=\tilde{\boldsymbol{g}}_t+\boldsymbol{n}_t/B, \\ \boldsymbol{w}_{t+1}^*&=\boldsymbol{w}_t-\eta\boldsymbol{g}_t^*.
	\end{aligned}
\end{equation}
Accordingly, the following definition establishes the measurement for model efficiency (ME). Obviously, a smaller ME means a better model efficiency.
\begin{definition}
	\textit{(Model Efficiency (ME)).} Suppose there exists a global optima $\boldsymbol{w}^{\star}$, the model deficiency can be measured by the Euclidean Distance between the current model $\boldsymbol{w}_{t+1}^*$ and the optima $\boldsymbol{w}^{\star}$, i.e.,
	\begin{equation}
		\mbox{Model efficiency (ME)}=\left\|\boldsymbol{w}_{t+1}^*-\boldsymbol{w}^{\star}\right\|^2.
	\end{equation}
\end{definition}
%Besides, it is essential to control the overall privacy cost in the whole training process. For Laplace mechanism, only composition theorem~\cite{kairouz2015composition} is applicable. For Gaussian mechanism, on the other hand, the moments accountant~\cite{abadi2016deep} is an economic method to measure the collective privacy cost.
%\begin{definition}
%	\textit{(Moments Accountant).} Suppose a model is trained for $T$ iterations with a batch size $B$ on a dataset of size $\vert D\vert$. For any $\epsilon<\frac{c_1B^2T}{N^2}$, there exists such constants $c_1$ and $c_2$ that the model is $(\epsilon,\delta)$-DP for any $\delta>0$ if we choose the noise multiplier $\sigma$ to be:
%	\begin{equation*}
%		\sigma \ge c_2\frac{B\sqrt{T\log(1/\delta)}}{\vert D\vert\epsilon}.
%	\end{equation*}
%\end{definition}
%To accumulate the overall privacy cost of the training process, composition theorem of RDP~\cite{kairouz2015composition} is applicable. 
As having to validate the optimality of GeoDP over DP on preserving the descent trend, we follow the common practice\cite{wang2019collecting} and adopt mean square error (MSE) to measure the error on perturbed directions. In general, a larger MSE means a larger perturbation.
\begin{definition}
	\textit{(Mean Square Error (MSE)).} Considering the perturbed directions $\left\{\boldsymbol{\theta}_1^*,\boldsymbol{\theta}_2^*,...,\boldsymbol{\theta}_{m-1}^*,\boldsymbol{\theta}_{m}^*\right\}$ and the original directions $\left\{\boldsymbol{\theta}_1,\boldsymbol{\theta}_2,...,\boldsymbol{\theta}_{m-1},\boldsymbol{\theta}_{m}\right\}$ of $m$ gradients, MSE of perturbed directions is defined as follows:
	\begin{equation}
			\text{MSE}(\boldsymbol{\theta}^*)=\frac{1}{m}\sum_{i=1}^{m}\Vert\boldsymbol{\theta}_i^*-\boldsymbol{\theta}_i\Vert^2_2.
	\end{equation}
\end{definition}
The problem in this work is to investigate the impact of DP noise $\boldsymbol{n}_t$ on the SGD efficiency, i.e.,  $\left\|\boldsymbol{w}_{t+1}^*-\boldsymbol{w}^{\star}\right\|^2$, and further optimize the model efficiency by reducing the noise on the direction of a gradient, i.e., reducing $\text{MSE}(\boldsymbol{\theta}^*)$.
	\section{Deficiency of DP-SGD: a gap between directional SGD and numerical DP}
\label{sec:analysis}
In this section, we identify an intrinsic deficiency in DP-SGD. Let the trained models of DP-SGD and non-private SGD be denoted by    $\boldsymbol{w}_{t+1}^*= \boldsymbol{w}_t-\eta\tilde{\boldsymbol{g}}_t^*$ and $\boldsymbol{w}_{t+1}=\boldsymbol{w}_t-\eta\tilde{\boldsymbol{g}}_t$, respectively.
%The reason for applying different learning rates (i.e., $\eta^*$ and $\eta$) is to demonstrate that a fine-tuned learning rate can counteract the impact of DP noise on the magnitude of a gradient.
The Euclidean distances between the current models and the global optima (i.e., $\left\|\boldsymbol{w}_{t+1}^*-\boldsymbol{w}^\star\right\|^2$ and $\left\|\boldsymbol{w}_{t+1}-\boldsymbol{w}^\star\right\|^2$) reflect the model efficiency of DP-SGD and non-private SGD, respectively. Apparently, the smaller this distance is, the better efficiency the model achieves. Their efficiency difference (ED) (i.e., $\left\|\boldsymbol{w}_{t+1}^*-\boldsymbol{w}^\star\right\|^2-\left\|\boldsymbol{w}_{t+1}-\boldsymbol{w}^\star\right\|^2$), on the other hand, can describe the impact of DP noise on the model efficiency, as presented by the following theorem.
\begin{theorem}
	(\textit{Impact of DP Noise on Model Efficiency}). Suppose $\boldsymbol{n}_\sigma$ follows a noise distribution with the standard deviation $\sigma\boldsymbol{I}$, ED can be measured as:
	\begin{equation}
		\begin{aligned}
		&\left\|\boldsymbol{w}_{t+1}^*-\boldsymbol{w}^\star\right\|^2-\left\|\boldsymbol{w}_{t+1}-\boldsymbol{w}^\star\right\|^2\\=&{\eta}^2\underbrace{\left(\frac{2C}{B}\langle \boldsymbol{n}_\sigma,\tilde{\boldsymbol{g}}_t\rangle+\frac{C^2\boldsymbol{n}_\sigma^2}{B^2}\right)}_{Item\quad A}+\frac{2\eta C}{B}\underbrace{\langle \boldsymbol{n}_\sigma,\boldsymbol{w}^\star-\boldsymbol{w}_t\rangle}_{Item\quad B}.
		\end{aligned}
	\end{equation}
\label{the:impact}
\end{theorem}
\begin{proof}
	For DP-SGD, we have:
	\begin{equation}
		\begin{aligned}
			\left\|\boldsymbol{w}_{t+1}^*-\boldsymbol{w}^\star\right\|^2&=\left\|\boldsymbol{w}_t-\boldsymbol{w}^\star-\eta\tilde{\boldsymbol{g}}_t^*\right\|^2\\&=\Vert \boldsymbol{w}_t-\boldsymbol{w}^\star\Vert^2+\eta^2\Vert\tilde{\boldsymbol{g}}_t^*\Vert^2+2\eta\langle\tilde{\boldsymbol{g}}_t^*,\boldsymbol{w}^\star-\boldsymbol{w}_t\rangle.
		\end{aligned}
		\label{equ:EffiDPSGD}
	\end{equation}
	While for SGD, we have:
	\begin{equation}
		\begin{aligned}
			\left\|\boldsymbol{w}_{t+1}-\boldsymbol{w}^\star\right\|^2&=\left\|\boldsymbol{w}_t-\boldsymbol{w}^\star-\eta\tilde{\boldsymbol{g}}_t\right\|^2\\&=\Vert \boldsymbol{w}_t-\boldsymbol{w}^\star\Vert^2+\eta^2\Vert\tilde{\boldsymbol{g}}_t\Vert^2+2\eta\langle\tilde{\boldsymbol{g}}_t,\boldsymbol{w}^\star-\boldsymbol{w}_t\rangle.
		\end{aligned}
		\label{equ:EffiSGD}
	\end{equation}
	Subtracting Equation \ref{equ:EffiSGD} from Equation \ref{equ:EffiDPSGD}, we have:
	\begin{equation}
		\begin{aligned}
			&\left\|\boldsymbol{w}_{t+1}^*-\boldsymbol{w}^\star\right\|^2-\left\|\boldsymbol{w}_{t+1}-\boldsymbol{w}^\star\right\|^2\\=&{\eta}^2\underbrace{\left(\Vert\tilde{\boldsymbol{g}}_t^*\Vert^2-\Vert\tilde{\boldsymbol{g}}_t\Vert^2\right)}_{Item\quad A}+2\eta\underbrace{\langle\tilde{\boldsymbol{g}}_t^*-\tilde{\boldsymbol{g}}_t,\boldsymbol{w}^\star-\boldsymbol{w}_t\rangle}_{Item\quad B}.
		\end{aligned}
		\label{equ:ED}
	\end{equation}
	Recall that $\boldsymbol{n}_t$ follows a noise distribution whose standard deviation is $C\sigma\boldsymbol{I}$. Suppose $\boldsymbol{n}_\sigma$ follows a noise distribution with the standard deviation $\sigma\boldsymbol{I}$, we have $\boldsymbol{n}_t=C\boldsymbol{n}_\sigma$. 
	For Item A:
	\begin{equation}
		\begin{aligned}
			\Vert\tilde{\boldsymbol{g}}_t^*\Vert^2-\Vert\tilde{\boldsymbol{g}}_t\Vert^2&=\left(\tilde{\boldsymbol{g}}_t^*-\tilde{\boldsymbol{g}}_t\right)\left(\tilde{\boldsymbol{g}}_t^*+\tilde{\boldsymbol{g}}_t\right)\\&=\boldsymbol{n}_t/B\left(2\tilde{\boldsymbol{g}}_t+\boldsymbol{n}_t/B\right)\\&=2\langle C\boldsymbol{n}_\sigma/B,\tilde{\boldsymbol{g}}_t\rangle+C^2\boldsymbol{n}_\sigma^2/B^2.
		\end{aligned}
		\label{equ:A}
	\end{equation}
	And for Item B:
	\begin{equation}
		\tilde{\boldsymbol{g}}_t^*-\tilde{\boldsymbol{g}}_t=\boldsymbol{n}_t/B=C\boldsymbol{n}_\sigma/B.
		\label{equ:B}
	\end{equation}
	Applying Equation \ref{equ:A} and \ref{equ:B} into Equation \ref{equ:ED}, we have:
	\begin{equation}
		\begin{aligned}
			&\left\|\boldsymbol{w}_{t+1}^*-\boldsymbol{w}^\star\right\|^2-\left\|\boldsymbol{w}_{t+1}-\boldsymbol{w}^\star\right\|^2\\=&{\eta}^2\underbrace{\left(2\langle C\boldsymbol{n}_\sigma/B,\tilde{\boldsymbol{g}}_t\rangle+C^2\boldsymbol{n}_\sigma^2/B^2\right)}_{Item\quad A}+2\eta C/B\underbrace{\langle \boldsymbol{n}_\sigma,\boldsymbol{w}^\star-\boldsymbol{w}_t\rangle}_{Item\quad B}.
		\end{aligned}
		\label{equ:proof}
	\end{equation}	
\end{proof}

In general, we wish the efficiency of DP-SGD closer to SGD, i.e., to make ED as close to zero as possible. This theorem coincides with many empirical findings in existing works. Item A, for example, shows that the introduction of DP noise would cause a bias to the global optima. That is, \textbf{DP-SGD cannot stably converges to the global optima, while sometimes reaching that point}, as proved by Corollary \ref{coro:optima}. This means that the model efficiency of DP-SGD is always lower than regular SGD~\cite{xia2023differentially,zhang2022understanding,10.5555/3495724.3496879,tang2024dp}. In practice, in order to provide a better model efficiency, existing works~\cite{abadi2016deep,yu2019make,feng2020privacy} apply lower noise scale (i.e., smaller $\boldsymbol{n}_\sigma$) when DP-SGD is about to converge. This operation makes Item A close to zero (but normally non-zero). Another example is that large batch size can enhance the efficiency of DP-SGD, as it can certainly reduce both Item A and Item B~\cite{fu2023dpsur}. 
\begin{corollary}
DP-SGD cannot stably stays at global optima.
\label{coro:optima}
\end{corollary}
\begin{proof}
	 Assume DP-SGD reaches the global optima at $t$-th iteration, i.e. $\boldsymbol{w}_t= \boldsymbol{w}^\star$, and apply this to Equation \ref{equ:proof} to have Equation \ref{equ:convergence} at $t+1$ iteration. Accordingly, Item B becomes zero while Item A is non-zero unless $\boldsymbol{n}_\sigma$ stays zero (which is generally negative). It proves that DP-SGD deviates from the global optima at $t+1$-th iteration even it can somehow reach it at $t$-th iteration.
	\begin{equation}
	\begin{aligned}
		&\lim\limits_{\boldsymbol{w}_t\to \boldsymbol{w}^\star}\left\|\boldsymbol{w}_{t+1}^*-\boldsymbol{w}^\star\right\|^2-\left\|\boldsymbol{w}_{t+1}-\boldsymbol{w}^\star\right\|^2\\=&{\eta}^2\underbrace{\left(\frac{2C}{B}\langle \boldsymbol{n}_\sigma,\tilde{\boldsymbol{g}}_t\rangle+\frac{C^2\boldsymbol{n}_\sigma^2}{B^2}\right)}_{Item\quad A}.
	\end{aligned}
	\label{equ:convergence}
	\end{equation}
\end{proof}

More importantly, this theorem reveals that DP-SGD techniques, such as adaptive clipping and learning rate, are incapable of counteracting the impact of DP noise on the direction of a gradient. On one hand, \textbf{Item A describes how the noise scale impacts the model efficiency}. To reduce this impact, small learning rate (${\eta}^2$) and clipping threshold ($C$ and $C^2$), or large batch size $B$ is effective. This conclusion is confirmed by many existing works, as reviewed in Section \ref{sec:literature}. On the other hand, \textbf{Item B}, the inner product between the noise $\boldsymbol{n}_t$ and the training process ($\boldsymbol{w}^\star-\boldsymbol{w}_t$ can be considered as the distance for SGD to descend, i.e., descent trend) \textbf{reflects how the perturbation impacts the further training}. While capable of reducing Item A, fine-tuning hyper-parameters cannot reduce Item B, as proved by the following corollary. 

\begin{corollary}
	Optimization techniques of DP-SGD (i.e., fine-tuning clipping and learning rate) cannot reduce the impact of noise on the gradient direction.
\label{coro:bound}
\end{corollary}
\begin{proof}
	We analyze the effectiveness of DP-SGD techniques (i.e., fine-tuning clipping, learning rate and batch size) on Item A and Item B, respectively.
	\begin{enumerate}
		\item \textit{Item A}.

		\textit{As per learning rate,} we apply different learning rate $\eta^*$ to DP-SGD, and see if tuning $\eta^*$ can make Item A zero. Applying $\eta^*$ to Equation \ref{equ:ED}, we have:
		\begin{equation}
			\text{Item A}=\eta^{*2}\Vert\tilde{\boldsymbol{g}}_t^*\Vert^2-\eta^2\Vert\tilde{\boldsymbol{g}}_t\Vert^2.
			\label{equ:boundmag}
		\end{equation}
		As Equation \ref{equ:boundmag} is only composed of numerical values, fined-tuned $\eta^*=\eta^2\Vert\tilde{\boldsymbol{g}}_t\Vert^2/\Vert\tilde{\boldsymbol{g}}_t^*\Vert^2$ can certainly zero Item A. 
		
		\textit{As for clipping}, given $\boldsymbol{n}_\sigma$ is a random variable drawn from the noise distribution whose standard deviation is $\sigma\boldsymbol{I}$, we have:
		\begin{equation}
			\boldsymbol{n}_t=C\boldsymbol{n}_\sigma.
			\label{equ:noisemultiplier}
		\end{equation} 
		As $\tilde{\boldsymbol{g}}_t^*=\tilde{\boldsymbol{g}}_t+\boldsymbol{n}_t/B$, reducing $C$ certainly reduces the scale of $\tilde{\boldsymbol{g}}_t^*$.
		Overall, fine-tuning of DP-SGD can certainly reduce Item A.
		\item \textit{Item B}.
		
		\textit{For learning rate}, we have:
		\begin{equation}
			\begin{aligned}
				\text{Item B}&=\langle\eta^*\tilde{\boldsymbol{g}}^*_t-\eta\tilde{\boldsymbol{g}}_t,\boldsymbol{w}^\star-\boldsymbol{w}_t\rangle\\&=\Vert\eta^*\tilde{\boldsymbol{g}}^*_t-\eta\tilde{\boldsymbol{g}}_t\Vert\Vert\boldsymbol{w}^\star-\boldsymbol{w}_t\Vert\cos\theta.				
			\end{aligned}
		\end{equation}
		where $\theta$ is the relative angle between two vectors. Apparently, no matter how to fine-tune $\eta^*$, how $\eta^*\tilde{\boldsymbol{g}}^*_t-\eta\tilde{\boldsymbol{g}}_t$ varies is rather random because there is no relevance between $\eta^*$ and $\eta^*\tilde{\boldsymbol{g}}^*_t-\eta\tilde{\boldsymbol{g}}_t$ as well as $\theta$.
		
		For clipping, we prove that it cannot change the geometric property of the perturbed gradient, although the noise scale is indeed changed. If the clipping thresholds $C_1$, $C_2$ and a gradient $\boldsymbol{g}(\Vert\boldsymbol{g}\Vert\ge C_1\ge C_2)$, we have the clipped gradient $\tilde{\boldsymbol{g}}_1=\frac{\boldsymbol{g}}{\Vert\boldsymbol{g}_1\Vert/C_1}$, $\tilde{\boldsymbol{g}}_2=\frac{\boldsymbol{g}}{\Vert\boldsymbol{g}_2\Vert/C_2}$ as per Equation \ref{equ:clipping} and corresponding noise $\boldsymbol{n}_1=C_1\boldsymbol{n}_\sigma$, $\boldsymbol{n}_2=C_2\boldsymbol{n}_\sigma$ as per Equation \ref{equ:noisemultiplier}. Accordingly, the perturbed gradient is:
		\begin{equation}
			\begin{aligned}
				\tilde{\boldsymbol{g}}_1^*&=\tilde{\boldsymbol{g}}_1+\boldsymbol{n}_1/B=\frac{\boldsymbol{g}}{\Vert\boldsymbol{g}_1\Vert/C_1}+C_1/B\boldsymbol{n}_\sigma.\\
				\tilde{\boldsymbol{g}}_2^*&=\tilde{\boldsymbol{g}}_2+\boldsymbol{n}_2/B=\frac{\boldsymbol{g}}{\Vert\boldsymbol{g}_2\Vert/C_2}+C_2/B\boldsymbol{n}_\sigma.
			\end{aligned}
		\end{equation}
		Then, we have:
		\begin{equation}
			\begin{aligned}
				\frac{\tilde{\boldsymbol{g}}_1^*}{C_1}=\frac{\tilde{\boldsymbol{g}}_2^*}{C_2},
				\Vert\tilde{\boldsymbol{g}}_1^*\Vert\ge\Vert\tilde{\boldsymbol{g}}_2^*\Vert.
			\end{aligned}
		\end{equation}
		Namely, clipping cannot control the directions of perturbed gradients $\frac{\tilde{\boldsymbol{g}}_1^*}{C_1}=\frac{\tilde{\boldsymbol{g}}_2^*}{C_2}$, while indeed reducing the noise scale ($\Vert\tilde{\boldsymbol{g}}_1^*\Vert\ge\Vert\tilde{\boldsymbol{g}}_2^*\Vert$).
	\end{enumerate}
\end{proof}
In general, this corollary points out a intrinsic deficiency of DP-SGD. That is, as a gradient is actually a vector instead of a numerical array, \textbf{traditional DP mechanisms}, which add noise to values of a gradient, \textbf{cannot directly reduce the noise on gradient direction (Item B)}. Even worse, \textbf{DP introduces biased noise to the direction, while adding unbiased noise to the gradient itself,} as further proved via hyper-spherical coordinate system (see Lemma \ref{lem:MSE} for rigorous proofs).
	\section{Geometric perturbation: GeoDP}
\label{sec:strategy}
In the previous analysis, we have proved the sub-optimality of traditional DP-SGD. In this section, we seize this opportunity to {\bf perturb the direction and the magnitude of a gradient, respectively, so that the noise on descent trend is directly reduced}. Within the DP framework, our strategy significantly improves the model efficiency.

In what follows, we first introduce $d$-spherical coordinate system \cite{thomas2006multivariable} in Section~\ref{subsec:hyper}, where one $d$-dimensional gradient is converted to one magnitude and one direction. By perturbing gradients in the $d$-spherical coordinate system, we propose our perturbation strategy \textit{GeoDP} to optimize the model efficiency in Section \ref{subsec:GeoDP}. Privacy and efficiency analysis is provided to prove its compliance with DP definition and huge advantages over DP-SGD in Section \ref{subsec:comparison}.

\subsection{Hyper-spherical Coordinate System}
\label{subsec:hyper}
The $d$-spherical coordinate system~\cite{thomas2006multivariable}, also known as the hyper-spherical coordinate system, is commonly used to analyze geometric objects in high-dimensional space, e.g., the gradient. Compared to the rectangular coordinate system \cite{thomas2006multivariable}, such a system directly represents any $d$-dimensional vector  $\boldsymbol{g}=(\boldsymbol{g}_1, \boldsymbol{g}_2, ..., \boldsymbol{g}_{d-1}, \boldsymbol{g}_d)$ using a magnitude $\left\|\boldsymbol{g}\right\|$ and a direction $\boldsymbol{\theta}=\left(\boldsymbol{\theta}_1,\boldsymbol{\theta}_2,...,\boldsymbol{\theta}_{d-2},\boldsymbol{\theta}_{d-1}\right)$. Formally, the magnitude is:
\begin{equation}
	\left\|\boldsymbol{g}\right\|=\sqrt{\sum_{z=1}^{d}\boldsymbol{g}_z^2}.
	\label{equ:SphMig}
\end{equation}
and its direction $\boldsymbol{\theta}$ is:
\begin{equation}
	\begin{aligned}
		\boldsymbol{\theta}_z=
		\begin{cases}
			\operatorname{arctan2}\left(\sqrt{\sum_{z}^{d-1}\boldsymbol{g}_{z+1}^2},\boldsymbol{g}_z\right) & \text{if } 1\leq z\leq d-2,\\
			\operatorname{arctan2}\left(\boldsymbol{g}_{z+1},\boldsymbol{g}_z\right)&\text{if } z= d-1.
		\end{cases}.
	\end{aligned}
	\label{equ:SphAng}
\end{equation}
where arctan2 is the two-argument arctangent function defined as follows:
\begin{equation}
\operatorname{arctan2}(y, x) =
\begin{cases}
	\arctan\left(\frac{y}{x}\right) & \text{if } x > 0, \\
	\arctan\left(\frac{y}{x}\right) +\pi  & \text{if } x < 0 \text{ and } y \ge 0, \\
	\arctan\left(\frac{y}{x}\right)-\pi & \text{if } x < 0 \text{ and } y < 0, \\
	\frac{\pi}{2} & \text{if } x = 0 \text{ and } y > 0, \\
	-\frac{\pi}{2} & \text{if } x = 0 \text{ and } y < 0, \\
	\text{undefined} & \text{if } x = 0 \text{ and } y = 0.
\end{cases}.
\label{equ:atan2}
\end{equation}

While having the same functionality as arctan, arctan2 is more robust. For example, arctan2 can deal with a zero denominator ($\boldsymbol{g}_z=0$). Note that $\sqrt{\sum_{z}^{d-1}\boldsymbol{g}_{z+1}^2}$ in Equation \ref{equ:SphAng} is always non-negative. For $1\leq z \leq d-2$, the range of $			\operatorname{arctan2}\left(\sqrt{\sum_{z}^{d-1}\boldsymbol{g}_{z+1}^2},\boldsymbol{g}_z\right)$ is either $\left(0,\frac{\pi}{2}\right]$ or $\left(\frac{\pi}{2},\pi\right)$ if $\boldsymbol{g}_z\ge 0$ or $\boldsymbol{g}_z< 0$, as per Equation \ref{equ:atan2}. \textbf{As such, the range of $\boldsymbol{\theta}_{1\leq z\leq d-2}$ is $\left(0,\pi\right)$. For $z=d-1$, the range of $\boldsymbol{\theta}_z$ is $\left(-\pi,\pi\right)$ as per Equation \ref{equ:atan2}}.

We can also convert a vector $\left(\left\|\boldsymbol{g}\right\|, \boldsymbol{\theta}\right)$ in $d$-spherical coordinates back to rectangular coordinates $(\boldsymbol{g}_1, \boldsymbol{g}_2, ..., \boldsymbol{g}_{d-1}, \boldsymbol{g}_d)$ using the following equation:
\begin{equation}
	\begin{aligned}
		\boldsymbol{g}_z =
		\begin{cases}
			\left\|\boldsymbol{g}\right\| \cos{\boldsymbol{\theta}_z},  & \mbox{if }\quad z=1 \\
			\left\|\boldsymbol{g}\right\| \prod_{i=1}^{z-1}\sin{\boldsymbol{\theta}_i} \cos{\boldsymbol{\theta}_z}, &\mbox{if }\quad 2 \leq z \leq d - 1\\
			\left\|\boldsymbol{g}\right\| \prod_{i=1}^{z-1}\sin{\boldsymbol{\theta}_i}, &\mbox{if }\quad z=d
		\end{cases}.
	\end{aligned}
	\label{equ:Rec}
\end{equation}
Figure \ref {fig:hyper} provides an example of conversions in three-dimensional space. Given $\left\|\boldsymbol{g}\right\|=\sqrt{\boldsymbol{g}_1^2+\boldsymbol{g}_2^2+\boldsymbol{g}_3^2}$, $\boldsymbol{\theta}_1=\operatorname{arctan2}\left(\sqrt{\boldsymbol{g}_2^2+\boldsymbol{g}_3^2},\boldsymbol{g}_1\right)$ and $\boldsymbol{\theta}_2=\operatorname{arctan2}\left(\boldsymbol{g}_3,\boldsymbol{g}_2\right)$, a vector $\boldsymbol{g}=(\boldsymbol{g}_1,\boldsymbol{g}_2,\boldsymbol{g}_3)$ in rectangular coordinate system (marked in black) can be represented as $(\left\|\boldsymbol{g}\right\|,\boldsymbol{\theta}_1,\boldsymbol{\theta}_2)$ in hyper-spherical coordinate system (marked in blue). Without loss of generality, we use $\boldsymbol{g}\leftrightarrow\left(\left\|\boldsymbol{g}\right\|, \boldsymbol{\theta}\right)$ to denote the reversible conversions between two systems.
\begin{figure}[htbp]
	\centering
	\includegraphics[width=0.7\linewidth]{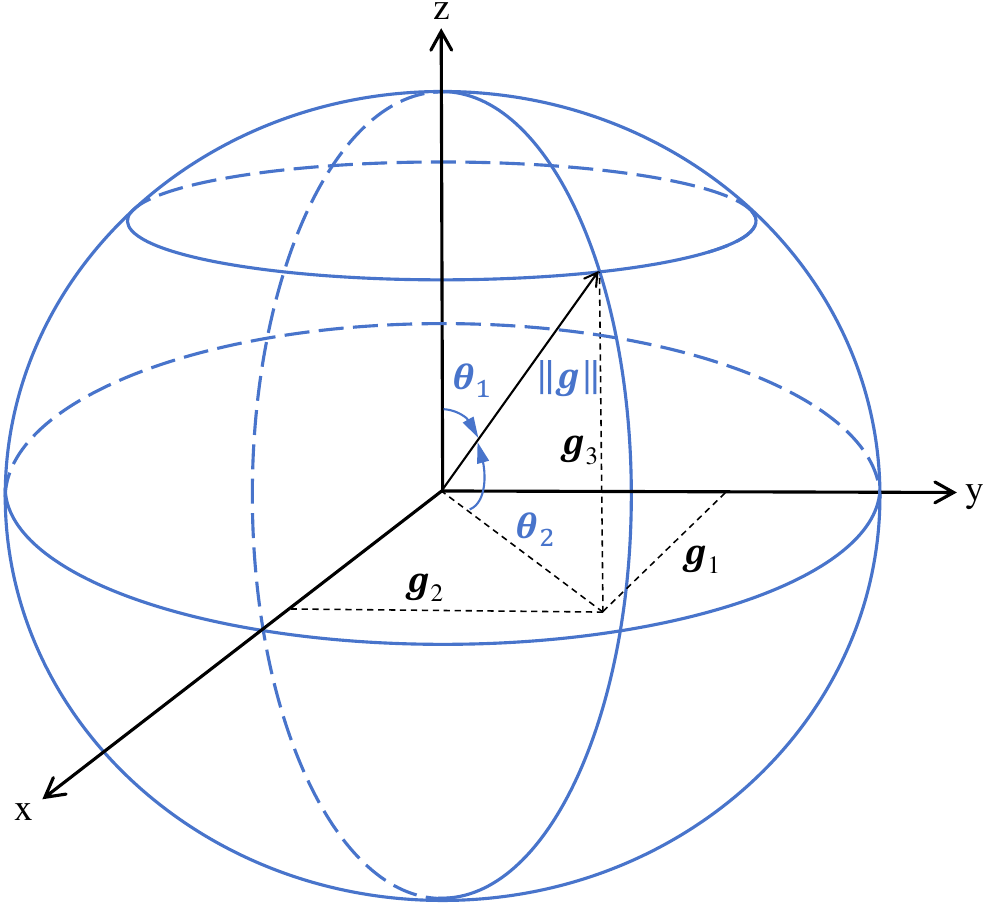}
	\caption{Coordinates Conversions in Three-dimensional Space}
	\label{fig:hyper}
\end{figure}
\subsection{GeoDP---Geometric DP Perturbation for DP-SGD}
\label{subsec:GeoDP}
\textit{GeoDP} directly reduces the noise on the descent trend via $d$-spherical coordinate system. Algorithm \ref{algo:GeoDP} describes how \textit{GeoDP} works, and major steps are interpreted as follows:
\begin{enumerate}
	\item \textit{Spherical-coordinate Conversion:} Convert the clipped gradient to hyper-spherical coordinate system according to Equation \ref{equ:SphMig} and Equation \ref{equ:SphAng}, i.e., $\boldsymbol{g}\rightarrow\left(\left\|\boldsymbol{g}\right\|, \boldsymbol{\theta}\right)$, which allows perturbation on the magnitude and the direction of a gradient, respectively.
	\item \textit{Reducing the Direction Range (Sensitivity):} According to Theorem \ref{the:ModelingDirection}, the averaged direction of gradients $\left\{\tilde{\boldsymbol{g}}_{tj}|1\leq j\leq B\right\}$ should be centered at one small range, rather than uniformly spreading the whole vector space. This conclusion is also confirmed by various SGD studies~\cite{yu2019make,bottou2012stochastic}. DP-SGD, taking the whole direction space as the privacy region, is therefore overprotective and low efficient. In this work, a bounding factor $\beta\in \left(0,1\right]$ defines the privacy region into a subspace around the original direction, which significantly reduces the noise addition in Step 3. For $1\leq z<d-1$, given $0\leq\Gamma_1\leq \boldsymbol{\theta}_z\leq \Gamma_2\leq\pi$, $\beta$ determines the range between $\Gamma_1$ and $\Gamma_2$, i.e., $\Gamma_2-\Gamma_1=\Delta \boldsymbol{\theta}_z=\beta\pi$. Similarly, $\Gamma_2-\Gamma_1=\Delta \boldsymbol{\theta}_z=2\beta\pi$ for $z=d-1$. Note that $\beta=1$ means the full space. This parameter directly determines the sensitivity of the direction, which consequently influences the noise addition in the following step. 
%	because $\boldsymbol{\theta}$ is essentially an array instead of a vector. 
%	Under this factor, this stage clips directions into $\tilde{\boldsymbol{\theta}}_{1\leq z\leq d-2} \in \left((1-\beta)\pi,\pi\right)$ while $\tilde{\boldsymbol{\theta}}_{d-1} \in \left(-\beta\pi,\beta\pi\right)$. 
%	Besides, in terms of $\frac{\sqrt{\sum_{z}^{d-1}\boldsymbol{g}_{z+1}^2}}{\boldsymbol{g}_z}$, we observe that the numerator is normally larger than the denominator. That is,  $\operatorname{arctan2}\left(\sqrt{\sum_{z}^{d-1}\boldsymbol{g}_{z+1}^2},\boldsymbol{g}_z\right)$ is more probably located on the right side of $(0,\pi)$. As such, we prefer the range $\left((1-\beta)\pi,\pi\right)$ for $\tilde{\boldsymbol{\theta}}_{1\leq z\leq d-2} $.
\item \textit{Noise Addition:} GeoDP allows to perturb the magnitude and the direction of a gradient, respectively. For the magnitude, $\left\|\tilde{\boldsymbol{g}}_t\right\|$ is already bounded by $C$ in the first stage. Similar to DP-SGD, the noise scale of the perturbed magnitude is $C\sigma$. For the direction, the noise scale is the sensitivity $\Delta\boldsymbol{\theta}$ times the noise multiplier $\sigma$. Note that maximum changes of $\tilde{\boldsymbol{\theta}}_{1\leq z\leq d-2}$ and $\tilde{\boldsymbol{\theta}}_{d-1}$ are $\beta\pi$ and $2\beta\pi$, respectively, due to the bounding of the direction range. Overall, $\Delta\boldsymbol{\theta}=\sqrt{(d-2)(\beta\pi)^2+(2\beta\pi)^2}=\sqrt{d+2}\beta\pi$.
\item \textit{Rectangular-coordinate Conversion:} Convert the perturbed magnitude and direction back to rectangular coordinates according to Equation \ref{equ:Rec}, i.e.,   $\left(\left\|\tilde{\boldsymbol{g}}_t\right\|^\star,\boldsymbol{\theta}_t^\star\right)\rightarrow \tilde{\boldsymbol{g}}_t^\star$, which allows future gradient descent.
\end{enumerate}
\begin{algorithm}
	\caption{GeoDP-SGD}
	\label{algo:GeoDP}
	\begin{algorithmic}[1]
		\Require Batch size $B$, noise multiplier $\sigma$, clipping threshold $C$, bounding factor $\beta (0<\beta\leq 1)$, learning rate $\eta$, total number of iterations $T$.
		\Ensure  Trained model $\boldsymbol{w}_T^\star$.
		\State Initialize a model with parameters $\boldsymbol{w}_0$.
		\label{item:init}
		\For{each iteration $t=0,1,...,T-2,T-1$ }
		\State Derive the average clipped gradient $\tilde{\boldsymbol{g}}_t$ with respect to the batch size $B$ and the clipping threshold $C$.
		\label{item:grad}
		\State Convert $\tilde{\boldsymbol{g}}_t$ to $d$-spherical coordinates as $\left(\left\|\tilde{\boldsymbol{g}}_t\right\|, \boldsymbol{\theta}_t\right)$.
		\label{item:Sph}
		\State
%		Clip $\boldsymbol{\theta}_t$ into $\tilde{\boldsymbol{\theta}}$ as follows:
%		\begin{equation*}
%			\begin{aligned}
%				\tilde{\boldsymbol{\theta}}=				
%				\begin{cases}
%					\tilde{\boldsymbol{\theta}}_{1\leq z \leq d-2}&=
%					\begin{cases}
%						\boldsymbol{\theta}_z &\mbox{if }\quad \boldsymbol{\theta}_z>(1-\beta)\pi,\\
%						(1-\beta)\pi  &\mbox{if }\quad \boldsymbol{\theta}_z\leq(1-\beta)\pi.
%					\end{cases}
%					\\\tilde{\boldsymbol{\theta}}_{d-1}&=
%					\begin{cases}
%						\boldsymbol{\theta}_z &\mbox{if }\quad \left\|\boldsymbol{\theta}_z\right\|<\beta\pi,\\
%						\beta\pi  &\mbox{if }\quad \boldsymbol{\theta}_z\ge\beta\pi,\\
%						-\beta\pi  &\mbox{if }\quad \boldsymbol{\theta}_z\leq-\beta\pi.\\
%					\end{cases}				
%				\end{cases}	
%			\end{aligned}
%		\end{equation*}
%		\label{item:bound}
		Bound the privacy region $\Delta$ of $\boldsymbol{\theta}$ as follows:
		\begin{equation*}
			\begin{aligned}
				\Delta\boldsymbol{\theta}_z=				
				\begin{cases}
					\Delta\boldsymbol{\theta}_{1\leq z \leq d-2}&=\beta\pi,
					\\\Delta\boldsymbol{\theta}_{d-1}&=2\beta\pi.			
				\end{cases}	
			\end{aligned}
		\end{equation*}
		\label{item:bound}
		\State
		$\left\|\tilde{\boldsymbol{g}}_t\right\|^\star=\left\|\tilde{\boldsymbol{g}}_t\right\|+\frac{C}{B}\boldsymbol{n}_\sigma$, $\tilde{\boldsymbol{\theta}}_t^\star=\tilde{\boldsymbol{\theta}}_t+\frac{\sqrt{d+2}\beta\pi}{B}\boldsymbol{n}_\sigma$, where $\boldsymbol{n}_\sigma$ follows a zero-mean Gaussian distribution with standard deviation $\sigma$.
		\label{item:Per}
		\State Convert $\left(\left\|\tilde{\boldsymbol{g}}_t\right\|^\star,\tilde{\boldsymbol{\theta}}_t^\star\right)$ back to rectangular coordinates as the perturbed gradient $\tilde{\boldsymbol{g}}_t^\star$.
		\State
		Update $\boldsymbol{w}_{t+1}^\star$ by taking a step in the direction of the noisy gradient, i.e., $\boldsymbol{w}_{t+1}^\star=\boldsymbol{w}_t-\eta\tilde{\boldsymbol{g}}_t^\star$.
		\EndFor
		\label{item::Rec}
	\end{algorithmic}
\end{algorithm}

In general, GeoDP provides better efficiency to SGD in two perspectives. First, \textbf{GeoDP adds unbiased noise, whereas traditional DP introduces biased perturbation, to the direction of a gradient} (see Lemma \ref{lem:MSE} for rigorous proofs). This counter-intuitive conclusion is supported by the fact that tradition DP, which adds unbiased noise to the gradient itself, however accumulates noise on different angles of one direction. Example \ref{exam:NoiseAccumulation} demonstrates how this noise accumulation happens. As such, numerical perturbation of DP seriously degrades the accuracy of directional information. GeoDP, on the other hand, independently controls the noise on each angle and therefore prevents noise accumulation.
\begin{example}
	Suppose we have a three-dimensional gradient $\boldsymbol{g}=(\boldsymbol{g}_1,\boldsymbol{g}_2,\boldsymbol{g}_3)$. Following traditional DP, these three should be added noise $\boldsymbol{n}=(\boldsymbol{n}_1,\boldsymbol{n}_2,\boldsymbol{n}_3)$. For the direction of this perturbed gradient $\boldsymbol{\theta}$, its first angle $\boldsymbol{\theta}_1$ should be $\operatorname{arctan2}\left(\sqrt{(\boldsymbol{g}_2+\boldsymbol{n}_2)^2+(\boldsymbol{g}_3+\boldsymbol{n}_3)^2},\boldsymbol{g}_1+\boldsymbol{n}_1\right)$, according to Equation 4. It is very obvious that noise of three dimensions ($\boldsymbol{n}_1,\boldsymbol{n}_2,\boldsymbol{n}_3$) is accumulated to the first angle $\boldsymbol{\theta}_1$, and this accumulation is biased.
\label{exam:NoiseAccumulation}
\end{example}

Second, via coordinates conversion, $d$-dimensional gradient is transferred to one magnitude and $d-1$ directions. By composition theory, $\frac{d-1}{d}$ privacy budget is allocated to the direction by GeoDP, which can better preserves directional information. 

Finally, we discuss the time complexity of GeoDP-SGD. For DP-SGD, given the size of private dataset $\vert D\vert$ and the number of gradient's dimensions $d$, DP-SGD takes $O(\vert D\vert d)$ time to calculate derivatives in one epoch \cite{yu2019make}. By contrast, coordinate conversions take $O(d)$ time to complete because it involves $d$-dimensional geometry calculation. Overall, GeoDP has the same time complexity $O(\vert D\vert d)$ as DP-SGD. 
\subsection{Comparison between GeoDP and Traditional DP: Efficiency and Privacy}
\label{subsec:comparison}
\subsubsection{Efficiency Comparison}
Via hyper-spherical coordinate system, we can identify deficiencies of traditional DP from a geometric perspective and further understand the merits of GeoDP. If clipping threshold is fixed, the max magnitude of a clipped gradient is determined, because $\left\|\tilde{\boldsymbol{g}}\right\|=\frac{\left\|\tilde{\boldsymbol{g}}\right\|}{\max \left\{1,\left\|\boldsymbol{g}\right\|/C\right\}}\leq C$. That is, the clipped gradients are within the hyper-sphere whose radius (abbreviated as $R$) is $C$. Figure \ref{fig:hyper} can help to understand this fact. For example, $\boldsymbol{g}$ (highlighted in black) in Figure \ref{fig:hyper} is vector within the hyper-sphere whose radius is $\left\|\boldsymbol{g}\right\|$ (highlighted in blue). By adding noise, traditional DP makes sure that any two gradients within the hyper-sphere are indistinguishable. However, there are two serious disadvantages. 

\textbf{On one hand, numerical noise addition does not respect the geometric property of gradients, as interpreted by the following example.} In general, traditional DP seriously sabotages the geometric property of a gradient, which eventually results in low model efficiency.
\begin{example}
	Suppose two parallel gradients $\tilde{\boldsymbol{g}}_1=(1,1)$, $\tilde{\boldsymbol{g}}_2=(2,2)$ and clipping threshold $C=2\sqrt{2}$. As such, these two gradients are all within $R=C=2\sqrt{2}$ hyper-sphere, and their directions are all $\boldsymbol{\theta}=\operatorname{arctan2}(1,1)=\operatorname{arctan2}(2,2)=\frac{\pi}{4}$. As such, DP adds the same scale of noise to both gradients for privacy preservation. Assuming that the noise $\boldsymbol{n}=(2,-1)$ is added to both gradients, directions of two perturbed gradients are $\boldsymbol{\theta}^*_1=\operatorname{arctan2}(1-1, 1+2)=0$ and $\boldsymbol{\theta}^*_2=\operatorname{arctan2}(2-1, 2+2)\approx \frac{2\pi}{25}$. Given parallel gradients ($\boldsymbol{\theta}=\frac{\pi}{4}$), directions of perturbed gradients ($\boldsymbol{\theta}^*_1\neq\boldsymbol{\theta}^*_2\neq\boldsymbol{\theta}$) are much different, even if the added noise ($\boldsymbol{n}=(2,-1)$) is the same. 	
\end{example}

\textbf{On the other hand, traditional DP, which preserves all directions within the hyper-sphere, actually adds excessive noise to the gradient.} Different from regular SGD, DP-SGD usually requires very large batch size (e.g., 16,384) to reduce the negative impact of noise~\cite{fu2023dpsur}, which makes training process less ``stochastic''~\cite{yu2019make,bottou2012stochastic}. In specific, the summation of gradients $\left\{\tilde{\boldsymbol{g}}_{jz}|1\leq j\leq B,1\leq z\leq d\right\}$ follows \textit{Lindeberg–Lévy Central Limit Theorem} (CLT)~\cite{SHANTHIKUMAR1984153} as these gradients are independently and identically distributed (each of them is derived from a single data of the same dataset). As such, we can use Gaussian distribution to model the average of this summation (i.e., $\tilde{\boldsymbol{g}}_z=\frac{1}{B}\sum_{j=1}^B\tilde{\boldsymbol{g}}_{jz}$), as proved by the following theorem. 
\begin{theorem} 
	(\textit{Modeling of the Averaged Stochastic Gradients}).
	Suppose that $var(\tilde{\boldsymbol{g}}_{jz})$ and $\mathbb{E}(\tilde{\boldsymbol{g}}_{jz})$ are the variance and the expectation of $\left\{\tilde{\boldsymbol{g}}_{jz}|1\leq j\leq B,1\leq z\leq d\right\}$, the probability density function (\textit{pdf}) of $\tilde{\boldsymbol{g}}_z$ is:
	\begin{equation}
		%		\lim\limits_{B\rightarrow\infty}\tilde{\boldsymbol{g}}\sim\mathcal{N}\left(\mathbb{E}( \tilde{\boldsymbol{g}}_{j}),\sqrt{\frac{var( \tilde{\boldsymbol{g}}_{j})}{B} }\right)\\
		\lim_{B\to \infty} f(\tilde{\boldsymbol{g}}_z)=\sqrt{\frac{B}{2\pi* var(\tilde{\boldsymbol{g}}_{jz})}}\exp\left({-\frac{B^2*\left(x-\mathbb{E}(\tilde{\boldsymbol{g}}_{jz})\right)^2}{2*var(\tilde{\boldsymbol{g}}_{jz})}}\right).
	\end{equation}
	\label{the:ModelingGradient}
\end{theorem}
\begin{proof}
	$\left\{\tilde{\boldsymbol{g}}_j|1\leq j\leq B\right\}$ are independently and identically distributed variables because each one is derived from one data $s_j$ of the same subset $S$. According to \textit{CLT}, the following probability holds:
	\begin{equation}
		\begin{aligned}
			&\lim\limits_{B\rightarrow\infty} \Pr\left(\frac{\sum_{j=1}^{B}  \tilde{\boldsymbol{g}}_{jz}-B*\mathbb{E}( \tilde{\boldsymbol{g}}_{jz})}{\sqrt{B*var( \tilde{\boldsymbol{g}}_{jz})}}\leq X\right)\\=&\lim\limits_{B\rightarrow\infty} \Pr\left(\frac{\frac{1}{B}\sum_{j=1}^{B}  \tilde{\boldsymbol{g}}_{jz}-\mathbb{E}( \tilde{\boldsymbol{g}}_{jz})}{\sqrt{var( \tilde{\boldsymbol{g}}_{jz})/B}}\leq X\right)=\int_{-\infty}^{X} \phi(x)dx.
		\end{aligned}
		\label{equ:NormalGradient}
	\end{equation}
	where $\phi(x)=\frac{1}{\sqrt{2\pi}} \exp(-\frac{x^2}{2})$ is the pdf of the standard Gaussian distribution. As such, $\frac{\sum_{j=1}^{B}  \tilde{\boldsymbol{g}}_{jz}/B-\mathbb{E}( \tilde{\boldsymbol{g}}_{jz})}{\sqrt{var( \tilde{\boldsymbol{g}}_{jz})/B}}$ follows standard gaussian distribution $\mathcal{N}(0,1)$, by which our claim is proved. 
\end{proof}
Indicated by Theorem \ref{the:ModelingGradient}, large batch size would incur unevenly distributed average of gradients, making the training process less stochastic. A further conjecture proposes that some directions within the space are also unlikely to be the direction of gradient descent at the current state, as proved by the following theorem. Suppose that 
%the magnitude of all gradients $\left\{\tilde{\boldsymbol{g}}_j|1\leq j\leq B\right\}$ to be $\left\|\tilde{\boldsymbol{g}}\right\|$, 
the directions of all gradients are $\left\{\boldsymbol{\theta}_{jz}|1\leq j\leq B,1\leq z\leq d\right\}$, we have:
\begin{theorem}
	(\textit{Modeling of the Averaged Directions of Gradients}). Suppose that $var(\tilde{\boldsymbol{\theta}}_{jz})$ and $\mathbb{E}(\tilde{\boldsymbol{\theta}}_{jz})$ are the variance and expectation of $\left\{\tilde{\boldsymbol{\theta}}_{jz}|1\leq j\leq B,1\leq z\leq d\right\}$, the \textit{pdf} of the averaged direction $\tilde{\boldsymbol{\theta}}_{z}=\frac{1}{B}\sum_{j=1}^{B} \tilde{\boldsymbol{\theta}}_{jz}$ is:
	\begin{equation}
		%		\lim\limits_{B\rightarrow\infty}\boldsymbol{\theta}\sim\mathcal{N}\left(\mathbb{E}(\boldsymbol{\theta}_{j}),\sqrt{\frac{var(\boldsymbol{\theta}_{j})}{B} }\right)
		\lim_{B\to \infty} f(\tilde{\boldsymbol{\theta}}_{z})=\sqrt{\frac{B}{2\pi* var(\tilde{\boldsymbol{\theta}}_{jz})}}\exp\left({-\frac{B^2*\left(x-\mathbb{E}(\tilde{\boldsymbol{\theta}}_{jz})\right)^2}{2*var(\tilde{\boldsymbol{\theta}}_{jz})}}\right).
	\end{equation}
	\label{the:ModelingDirection}
\end{theorem}
\begin{proof}
	%	First, we propose that the averaged direction $\boldsymbol{\theta}$ is the summation of respective direction $\boldsymbol{\theta}_j$, as proved by the following mathematical induction. If $z=2$, we have $\boldsymbol{\theta}^1=\sum_{j=1}^{z}\boldsymbol{\theta}_j=\boldsymbol{\theta}_1+\boldsymbol{\theta}_2$, because both vectors $\tilde{\boldsymbol{g}}_1$ and $\tilde{\boldsymbol{g}}_2$ have the same magnitude $\left\|\tilde{\boldsymbol{g}}\right\|$. If $z=k(k<B)$, suppose $\boldsymbol{\theta}^k=\sum_{j=1}^{k}\boldsymbol{\theta}_j$. If $z=k+1$, we have:
	%	\begin{equation}
		%		\boldsymbol{\theta}^{k+1}=\boldsymbol{\theta}^k+\boldsymbol{\theta}_{k+1}=\sum_{j=1}^{k+1}\boldsymbol{\theta}_j.
		%	\end{equation}
	%	by which our claim is proved.
	\begin{equation}
		\begin{aligned}
			&\lim\limits_{B\rightarrow\infty} \Pr\left(\frac{\sum_{j=1}^{B}  \tilde{\boldsymbol{\theta}}_{j}-B*\mathbb{E}( \tilde{\boldsymbol{\theta}}_{j})}{\sqrt{B*var( \tilde{\boldsymbol{\theta}}_{j})}}\leq X\right)\\=&\lim\limits_{B\rightarrow\infty} \Pr\left(\frac{\frac{1}{B}\sum_{j=1}^{B}  \tilde{\boldsymbol{\theta}}_{j}-\mathbb{E}( \tilde{\boldsymbol{\theta}}_{j})}{\sqrt{var( \tilde{\boldsymbol{\theta}}_{j})/B}}\leq X\right)=\int_{-\infty}^{X} \phi(x)dx.
		\end{aligned}		
	\end{equation}	
	where $\phi(x)=\frac{1}{\sqrt{2\pi}} \exp(-\frac{x^2}{2})$ is the pdf of the standard Gaussian distribution. As such, $\frac{\sum_{j=1}^{B}  \boldsymbol{\theta}_{j}/B-\mathbb{E}( \boldsymbol{\theta}_{j})}{\sqrt{var( \boldsymbol{\theta}_{j})/B}}$ follows standard gaussian distribution $\mathcal{N}(0,1)$, by which our claim is proved. 
\end{proof}
This theorem proves that the averaged direction of stochastic gradients actually concentrated at a certain direction, rather than spreading in the whole vector space. As such, traditional DP-SGD, only effective in the whole vector space, actually wastes privacy budgets to preserve unnecessary directions. In contrast, GeoDP preserves the subspace where directions of various gradients are concentrated, and therefore provides much better efficiency, as jointly proved by the following lemma (which indicates the better accuracy of GeoDP on preserving directional information) and theorem (which further indicates the superiority of GeoDP on model efficiency). Experimental results in Section \ref{subsec:perang} also confirm our analysis.
\begin{lemma}
	Given the original direction $\boldsymbol{\theta}$, two perturbed directions $\boldsymbol{\theta}^\star$ and $\boldsymbol{\theta}^*$ from GeoDP and DP, respectively, there always exists such a bounding factor $\beta$ that $\text{MSE}(\tilde{\boldsymbol{\theta}}_t^\star)<\text{MSE}(\tilde{\boldsymbol{\theta}}_t^*)$ holds. 
	\label{lem:MSE}
\end{lemma}
\begin{proof}
	For traditional DP (adding noise $\boldsymbol{n}$ to the gradient $\boldsymbol{g}$), we can derive the perturbed angle $\boldsymbol{\theta}^*_z$ according to Equation \ref{equ:SphAng}, i.e., 
	\begin{equation}
		\begin{aligned}
		&\boldsymbol{\theta}^*_z=\\
		&\begin{cases}
		\operatorname{arctan2}\left(\sqrt{\sum_{z}^{d-1}(\boldsymbol{g}_{z+1}+\boldsymbol{n}_{z+1})^2},\boldsymbol{g}_z+\boldsymbol{n}_z\right) & \text{if } 	1\leq z\leq d-2,\\
		\operatorname{arctan2}\left(\boldsymbol{g}_{z+1}+\boldsymbol{n}_{z+1},\boldsymbol{g}_z+\boldsymbol{n}_z\right)&\text{if } z= d-1.
		\end{cases}.			
		\end{aligned}
	\end{equation}
	Observing both acrtan2 equations above, we can conclude that the \textbf{traditional DP perturbation} introduces \textbf{biased} noise to the original direction, i.e., $\mathbb{E}(\boldsymbol{\theta}^*)\ne \boldsymbol{\theta} (bias(\boldsymbol{\theta}^*)\ne 0)$. Also, the variance of $\boldsymbol{\theta}$ ($var(\boldsymbol{\theta}^*)$) is non-zero, if the noise scale $\boldsymbol{n}_\sigma>0$. 
	
	For GeoDP, we have $\boldsymbol{\theta}^\star=\boldsymbol{\theta}+\frac{\sqrt{d+2}\beta\pi}{B}\boldsymbol{n}_\sigma$. Accordingly, $\mathbb{E}(\boldsymbol{\theta}^\star)=\mathbb{E}(\boldsymbol{\theta}+\frac{\sqrt{d+2}\beta\pi}{B}\boldsymbol{n}_\sigma)=\boldsymbol{\theta} (bias(\boldsymbol{\theta}^\star)=0)$, which means that GeoDP adds unbiased noise to the direction. Besides, $beta$ directly controls the noise added to the direction. In specific, the variance of $\boldsymbol{\theta}^\star (var(\boldsymbol{\theta}^\star))$ can approaching zero if $\beta \to 0$, because $\boldsymbol{\theta}^\star=\boldsymbol{\theta}+\frac{\sqrt{d+2}\beta\pi}{B}\boldsymbol{n}_\sigma$ approaches $0$ if $\beta \to 0$.
	
	Given that $\text{MSE}(\boldsymbol{\theta})=bias^2(\boldsymbol{\theta})+var(\boldsymbol{\theta})$~\cite{duan2024ldptube}, there always exist such one $\beta$ that: 
	\begin{equation}
		\text{MSE}(\boldsymbol{\theta}^\star)=bias^2(\boldsymbol{\theta}^\star)+var(\boldsymbol{\theta}^\star)<=bias^2(\boldsymbol{\theta}^*)+var(\boldsymbol{\theta}^*)=\text{MSE}(\boldsymbol{\theta}^*).
	\end{equation}
	by which our claim is proven.
\end{proof}
Supported by this lemma, we further prove the optimality of GeoDP to tradition DP in the efficiency of SGD tasks in the next theorem.
\begin{theorem}
	(\textit{Optimality of GeoDP}). Let $\boldsymbol{w}_{t+1}^\star= \boldsymbol{w}_t-\eta\tilde{\boldsymbol{g}}_t^{\star}$, $\boldsymbol{w}_{t+1}^*= \boldsymbol{w}_t-\eta\tilde{\boldsymbol{g}}_t^*$ and $\tilde{\boldsymbol{g}}_t$, $\tilde{\boldsymbol{g}}_t^\star$ and $\tilde{\boldsymbol{g}}_t^*$ be the clipped gradient, noisy gradients of GeoDP and DP, respectively. Besides, $\tilde{\boldsymbol{g}}_t\rightarrow\left(\left\|\tilde{\boldsymbol{g}}_t\right\|, \tilde{\boldsymbol{\theta}}_t\right)$, $\tilde{\boldsymbol{g}}_t^\star\rightarrow\left(\left\|\tilde{\boldsymbol{g}}_t\right\|^\star, \tilde{\boldsymbol{\theta}}_t^\star\right)$ and $\tilde{\boldsymbol{g}}_t^*\rightarrow\left(\left\|\tilde{\boldsymbol{g}}_t\right\|^*, \tilde{\boldsymbol{\theta}}_t^*\right)$. The following inequality always holds if $\tilde{\boldsymbol{g}}_t^\star$ and $\tilde{\boldsymbol{g}}_t^*$ both follow $(\epsilon,\delta)$-DP:
	\begin{equation}
		\mathbb{E}\left(\left\|\boldsymbol{w}_{t+1}^\star-\boldsymbol{w}^\star\right\|^2\right)<\mathbb{E}\left(\left\|\boldsymbol{w}_{t+1}^*-\boldsymbol{w}^\star\right\|^2\right).
	\end{equation}
	\label{the:optimality}
\end{theorem}
\begin{proof}
	Following Corollary \ref{coro:bound}, we just have to prove Item B of GeoDP is smaller than Item A of DP. Different learning rates $\eta^\star$ and $\eta^*$ are applied to GeoDP and DP, respectively.
	Recall from Corollary \ref{coro:bound}, we have:
	\begin{equation}
		\begin{aligned}
			\text{Item B}&=\langle\eta^\star\tilde{\boldsymbol{g}}^\star_t-\eta\tilde{\boldsymbol{g}}_t,\boldsymbol{w}^\star-\boldsymbol{w}_t\rangle\\&=\underbrace{\Vert\eta^\star\tilde{\boldsymbol{g}}^\star_t-\eta\tilde{\boldsymbol{g}}_t\Vert}_C\underbrace{\Vert\boldsymbol{w}^\star-\boldsymbol{w}_t\Vert}_D\underbrace{\cos\theta}_E.		
		\end{aligned}
	\end{equation}
	Note that the only way to optimize Item B is via Item C. Most likely, Item $D$, as the distance between the current model and the optima, is fixed, and Item E, which describes the relative angle between noise and the fixed distance, is too random to handle. Therefore, we manage to zero Item C as much as possible to optimize Item B. In general, we have:
	\begin{equation}
		\text{Item C}^2=\left(\eta^\star\tilde{\boldsymbol{g}}^\star_t\right)^2+\left(\eta\tilde{\boldsymbol{g}}_t\right)^2-2\eta^\star\eta\langle\tilde{\boldsymbol{g}}^\star_t,\tilde{\boldsymbol{g}}_t\rangle.
	\end{equation}
	
	While $\left(\eta^\star\tilde{\boldsymbol{g}}^\star_t\right)^2+\left(\eta\tilde{\boldsymbol{g}}_t\right)^2$ can be fine-tuned to zero by learning rates, the only way for $\langle\tilde{\boldsymbol{g}}^\star_t,\tilde{\boldsymbol{g}}_t\rangle$ to be zero is that the direction of $\boldsymbol{g}^\star$ should approximate that of $\tilde{\boldsymbol{g}}_t$ (or the opposite direction of $\tilde{\boldsymbol{g}}_t$, which rarely happens and is therefore out of question here.). Due to $\text{MSE}(\tilde{\boldsymbol{\theta}}_t^\star)<\text{MSE}(\tilde{\boldsymbol{\theta}}_t^*)$ in Lemma \ref{lem:MSE}, GeoDP can therefore more easily make Item B zero than DP, by which our claim is proved.
\end{proof}
\subsubsection{Privacy Comparison}
Now that the superiority of GeoDP on model efficiency is rigorously analyzed, we next prove its alignment with the formal DP definition. The following lemme and theorem analyze the privacy level of perturbed gradient direction and gradient itself of GeoDP, respectively. 
\begin{lemma}
	The perturbed direction from GeoDP $\tilde{\boldsymbol{\theta}}^\star$ under $\beta$ bounding factor satisfies $(\epsilon,\delta+\delta')-DP$, where
	\begin{equation}
		1-\int_{0}^{2\beta \pi}\underbrace{\int_{0}^{\beta \pi}...\int_{0}^{\beta \pi}}_{d-1} \prod_{z=1}^{d}f(\tilde{\boldsymbol{\theta}}_{z})d\tilde{\boldsymbol{\theta}}_{z}\leq\delta'\leq 1-\beta.
	\end{equation}
\end{lemma}
\begin{proof}
	While $\delta$ covers the probability where the strict DP is ineffective~\cite{dwork2006our,dwork2010differential,dwork2014algorithmic}, we use $\delta'$ to denote the probability of space where $(\epsilon,\delta)$-DP is ineffective. Since $\tilde{\boldsymbol{\theta}}^\star$ is generally not the expectation of $\left\{\boldsymbol{\theta}_j\right\}$, we have:
	\begin{equation}
		\delta'\ge 1-\int_{0}^{2\beta \pi}\underbrace{\int_{0}^{\beta \pi}...\int_{0}^{\beta \pi}}_{d-1} \prod_{z=1}^{d}f(\tilde{\boldsymbol{\theta}}_{z})d\tilde{\boldsymbol{\theta}}_{z}.
	\end{equation}
	Meanwhile, the space that $\beta$ cannot cover is $1-\beta$ if the directions are evenly distributed (as discussed before, they are not). As such, $\delta'\leq 1-\beta$, by which our claim is proved.
\end{proof}
\begin{theorem}
	(\textit{Privacy Level of GeoDP}).
	Given $\tilde{\boldsymbol{g}}\leftrightarrow\left(\left\|\tilde{\boldsymbol{g}}\right\|, \tilde{\boldsymbol{\theta}}\right)$, $\tilde{\boldsymbol{g}}^\star$ satisfies $\left(\epsilon,\delta+\delta'\right)$-DP if $\left\|\tilde{\boldsymbol{g}}\right\|^\star$ and $\tilde{\boldsymbol{\theta}}^\star$ follow $(\epsilon,\delta)$-DP and $(\epsilon,\delta+\delta')$-DP, respectively. 
	\label{the:PrivacyCostOfDP}
\end{theorem}
\begin{proof}
	Given two neighboring dataset $D$, $D'$ and their output sets $(\tilde{\boldsymbol{g}}^{\star},\tilde{\boldsymbol{\theta}}^{\star})=\left\{(\tilde{\boldsymbol{g}}_1^{\star},\tilde{\boldsymbol{\theta}}_1^{\star}), ...\right\}$ of $\mathcal{M}(D)$,  $(\tilde{\boldsymbol{g}}^{\star'},\tilde{\boldsymbol{\theta}}^{\star'})=\left\{(\tilde{\boldsymbol{g}}_1^{\star'},\tilde{\boldsymbol{\theta}}_1^{\star'}), ...\right\}$ of $\mathcal{M}(D')$, respectively, we have:
	\begin{equation}
		\begin{aligned}
			&\Pr[\mathcal{M}(D) \in S]=\Pr[(\tilde{\boldsymbol{g}}^\star,\tilde{\boldsymbol{\theta}}^\star)\in S] \\\leq&  \left(e^{\epsilon}\Pr[(\tilde{\boldsymbol{g}}^{\star'},\tilde{\boldsymbol{\theta}}^{\star})\in S]+\delta\right) \vee \left(e^{\epsilon}\Pr[(\tilde{\boldsymbol{g}}^{\star},\tilde{\boldsymbol{\theta}}^{\star'})\in S]+\delta+\delta'\right)\\=&  \left(e^{\epsilon}\Pr[(\tilde{\boldsymbol{g}}^{\star'},\tilde{\boldsymbol{\theta}}^{\star'})\in S]+\delta+\delta'\right)\\=&e^{\epsilon }\Pr[\mathcal{M}(D') \in S]+\delta+\delta'.
		\end{aligned}
	\end{equation}
	by which this theorem is proven.
\end{proof}
Compared with traditional DP which imposes $(\epsilon,\delta)$-DP on the whole gradient, GeoDP relieves the privacy level of gradient direction (i.e., $\tilde{\boldsymbol{\theta}}^\star$ satisfies $(\epsilon,\delta+\delta')$-DP) while maintaining the same privacy preservation on gradient magnitude (i.e., $\tilde{\boldsymbol{g}}^\star$ satisfies $(\epsilon,\delta)$-DP). In return, the model efficiency of SGD is much improved under the same noise scale. While the privacy preservation is weaker, GeoDP imposes more perturbation on gradient magnitude, making it even harder for various attacks to succeed. 
	\section{Experimental results}
\label{sec:experiment}
%This section empirically evaluates our analysis as well as the perturbation strategy \textit{GeoDP} in various learning tasks. 
%In Section \ref{subsec:perang}, we first validate Lemma \ref{lem:MSE} where GeoDP preserves directional information better than traditional DP. 
%We then compare performances of GeoDP with traditional DP in one machine learning model (i.e., Logistic Regression) and two deep learning models (i.e., CNN and ResNet) in Section \ref{subsec:logi} and Section \ref{subsec:CNN}, respectively. 
%Since GeoDP only modifies the way to perturb, instead of the target to be perturbed and the training process, existing optimization techniques, such as adaptive clipping and other advanced optimizers, are orthogonal to GeoDP. To demonstrate generality of GeoDP, we also compare the performance of GeoDP and DP in CNN with a state-of-the-art clipping technique AUTO-S~\cite{bu2024automatic}. Finally, to evaluate the practical privacy risk, we implement a benchmark membership inference attack (MIA) on both GeoDP and DP. 

\subsection{Experimental Setup}
\label{subsec:setup}
We conduct our experiments on a server with Intel Xeon Silver 4210R CPU, 128G RAM, and Nvidia GeForce RTX 3090 GPU on Ubuntu 20.04 LTS system. All results are repeated $100$ times to obtain the average. Unless otherwise specified, we fix $C=0.1$. %and $\delta=10^{-5}$. 
\subsubsection{Datasets and Models}
For model efficiency, we use two prevalent benchmark datasets, MNIST~\cite{lecun1998gradient} and CIFAR-10~\cite{krizhevsky2009learning}.
%For MIA tests, we adopt the state-of-the-art benchmark ML-DOCTOR~\cite{shokri2017membership} and use four public datasets, MNIST, CIFAR-10, CelebA~\cite{liu2015deep}, and FMNIST~\cite{xiao2017fashion}. 
Besides, we also conduct a standalone experiment to verify that GeoDP  preserves directional information better than DP (Lemma \ref{lem:MSE}). Due to the lack of public gradient datasets, we form a synthetic one for this experiment. The details of these datasets are as below. 

\noindent
\textbf{MNIST.} This is a dataset of 70,000 gray-scale images (28x28 pixels) of handwritten digits from 0 to 9, commonly used for training and testing machine learning algorithms in image recognition tasks. It consists of 60,000 training images and 10,000 testing images, with an even distribution across the 10 digit classes.

\noindent
\textbf{CIFAR-10.} It is a dataset of 60,000 small (32x32 pixels) color images, divided into 10 distinct classes such as animals and vehicles, used for machine learning and computer vision tasks. It contains 50,000 training images and 10,000 testing images, with each class having an equal number of images.

\noindent
\textbf{Synthetic Gradient Dataset.} To synthesize a dataset of gradients, we randomly collect $450,000$ gradients (of $20,000$ dimensions) from $9$ epochs of training a non-DP CNN ($B=1$) on CIFAR-10 (i.e., $50,000$ training images). Dimensions are randomly chosen in various experiments.

As for models, recall that our experiments aim to confirm the superiority of GeoDP to DP on SGD, instead of yearning the best empirical accuracy over all existing ML models. As such, we believe prevalent models such as LR, 2-layer CNN with Softmax activation and ResNet with 3 residual block (each one containing 2 convolutional layers and 1 rectified linear unit (ReLU)) are quite adequate to confirm the effectiveness of our strategy. 

%While a model's efficiency under DP can be improved by optimizing various factors, e.g., global training strategy~\cite{fu2023dpsur}, fine-tuning parameters~\cite{zhao2020dp}, advanced optimizer~\cite{tang2024dp} and model architecture~\cite{sander2023tan}, GeoDP is the first to focus on perturbation strategy, and makes no change to the logic of SGD so as to maintain its universal compatibility with any optimization that integrates SGD. As such, we believe basic models such as LR and CNN are more appropriate, as their interpretability can more accurately match the setting in our theoretical derivation. Such strategy has also been adopted in a few other theoretical works, such as~\cite{li2019convergence}, which only utilizes LR to verify the convergence analysis. 
\subsubsection{Competitive Methods}
As GeoDP is orthogonal to existing optimization techniques as interpreted in Section \ref{subsec:DPSGD}, we do not directly compare them. Instead, we compare GeoDP with DP on regular SGD from various perspectives, i.e., model efficiency, compatibility with existing optimization techniques. To demonstrate the generality of GeoDP, we also incorporate two state-of-the-art iterative optimization techniques, IS~\cite{wei2022dpis} and SUR~\cite{fu2024dpsur}, as well as two advanced clipping optimization techniques, AUTO-S~\cite{bu2024automatic} and PSAC~\cite{xia2023differentially}, to observe their improvements on GeoDP.

\subsection{GeoDP vs. DP: Accuracy of Descent Trend}
\label{subsec:perang}
On the synthetic dataset, we perturb gradients by GeoDP and DP, respectively, and compare their MSEs under various parameters. As illustrated in Figure \ref{fig:AngleCompare}, labels $\theta$ and $g$ represent MSEs of perturbed directions and gradients, respectively. In Figure \ref{subfig:Sigma}-\ref{subfig:Sigma2}, we fix dimension $d=5,000$ and batch size $B=2,048$, while varying noise multiplier $\sigma$ in $\left\{10^{-4},10^{-3},10^{-2},10^{-1},1,10\right\}$ (i.e., varying privacy budget $\epsilon$ in $\left\{484.5,153.2,48.5,15.3,4.9,1.5\right\}$
if $\delta=10^{-5}$) under three bounding factors $\beta=\left\{0.01,0.1,1\right\}$, respectively. We have two major observations. First, GeoDP better preserves directions (the red line is below the black line) while DP better preserves gradients (the blue line is  below the green line) in most scenarios. Second, GeoDP is sometimes not robust to large noise multiplier and high dimensionality. When $\sigma>1$ in Figure \ref{subfig:Sigma}, GeoDP is instead outperformed by DP in preserving directions. Similar results can be also observed in Figure \ref{subfig:Dimension}-\ref{subfig:Dimension2} (fixing $\sigma=8,B=4096$ while varying dimensionality in $\left\{500, 1000, 2000, 5000, 10000, 20000\right\}$) and Figure \ref{subfig:Batch}-\ref{subfig:Batch2} (fixing $d=10000,\sigma=8$ while varying batch size in $\left\{512,1024,2048,4096,8192,163984\right\}$), respectively. For example, Figure \ref{subfig:Dimension} and Figure \ref{subfig:Batch}, which all fix $\beta=1$, show that GeoDP is outperformed by DP on preserving directions when $d>2000$ and $B<8192$, respectively.

Before addressing this problem, we discuss reasons behind the ineffectiveness of GeoDP. Recall from Section \ref{subsec:GeoDP} that the perturbation of GeoDP on directions is $\frac{\sqrt{d+2}\beta\pi}{B}\boldsymbol{n}_\sigma$. Obviously, both large noise multiplier ($\boldsymbol{n}_\sigma$) and high dimensionality ($\sqrt{d+2}$) increase the perturbation on directions.

Nevertheless, GeoDP can overcome this shortcoming by tuning $\beta$, which controls the sensitivity of direction. In both Figures \ref{subfig:Sigma1} ($\beta=0.1$) and \ref{subfig:Sigma2} ($\beta=0.01$), we reduce the noise on the direction by reducing the bounding factor, and the pay-off is very significant. Results show that GeoDP simultaneously outperforms DP in both direction and gradient. Tuning $\beta$ is also effective in Figure \ref{subfig:Dimension1}, \ref{subfig:Dimension2} and Figure \ref{subfig:Batch1}, \ref{subfig:Batch2}, respectively. Most likely, smaller bounding factor reduces noise added to the direction while does not affect the noisy magnitude. Accordingly, GeoDP reduces both MSEs of direction and gradient, and thus perfectly outperforms DP in preserving directional information. 

To further confirm this conjecture, extensive experiments, by varying the bounding factor in $\left\{0.1,0.2,0.4,0.6,0.8,1.0\right\}$ under different scenarios, are conducted in Figure \ref{fig:bound}. All experimental results show that there always exists a bounding factor ($\beta=0.2$ in Figure \ref{subfig:Beta} and $\beta=0.4$ in Figure \ref{subfig:Beta1}
%and $\beta=0.8$ in Figure \ref{subfig:Beta2}) 
for GeoDP to outperform DP in preserving both direction and gradient. \textbf{These results also perfectly align with our theoretical analysis in Lemma \ref{lem:MSE} and Theorem \ref{the:optimality}, respectively.}

Also, GeoDP can improve accuracy by tuning batch size. As illustrated in Figure \ref{subfig:Batch} ($d=10000,\sigma=8,\beta=1$), we demonstrate how the performance of GeoDP is impacted by batch size. Obviously, a large batch size can boost GeoDP to provide optimal accuracy on directions. In contrast, the accuracy of DP on directions hardly changes with batch size (see the black line in \ref{subfig:Batch}), although the noise scale on gradients is reduced by larger batch size (see the blue line in \ref{subfig:Batch}). These results validate that \textbf{optimization techniques of DP-SGD}, such as fine-tuning learning rate, clipping threshold and batch size, \textbf{cannot reduce the noise on the direction, as confirmed by Corollary \ref{coro:bound}.} 
\begin{figure*}
	\begin{center}
		\subfigure[$d=5000, B=2048, \beta=1$]{
			\includegraphics[width=0.28\linewidth]{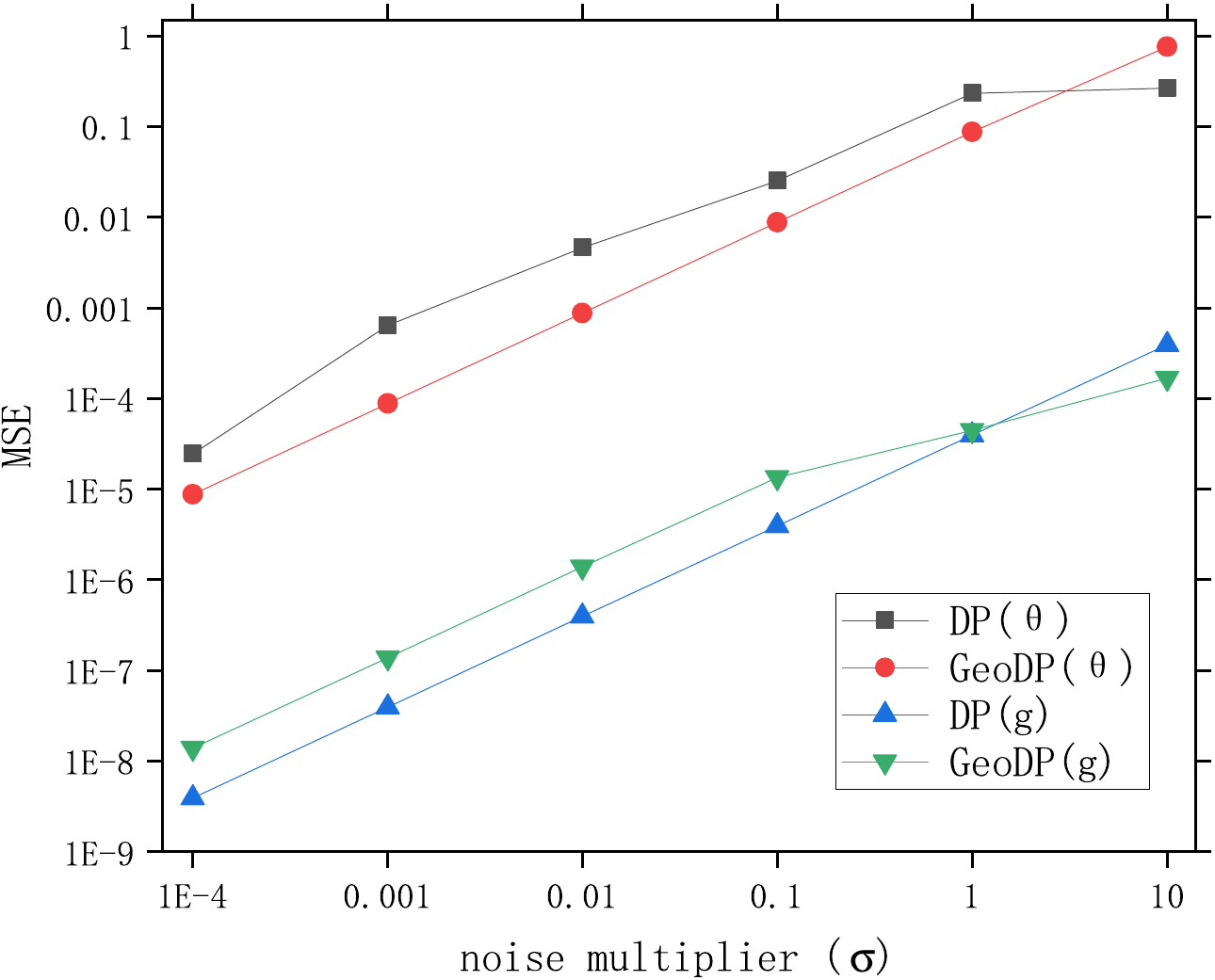}
			\label{subfig:Sigma}
		}
		\subfigure[$d=5000, B=2048, \beta=0.1$]{
			\includegraphics[width=0.28\linewidth]{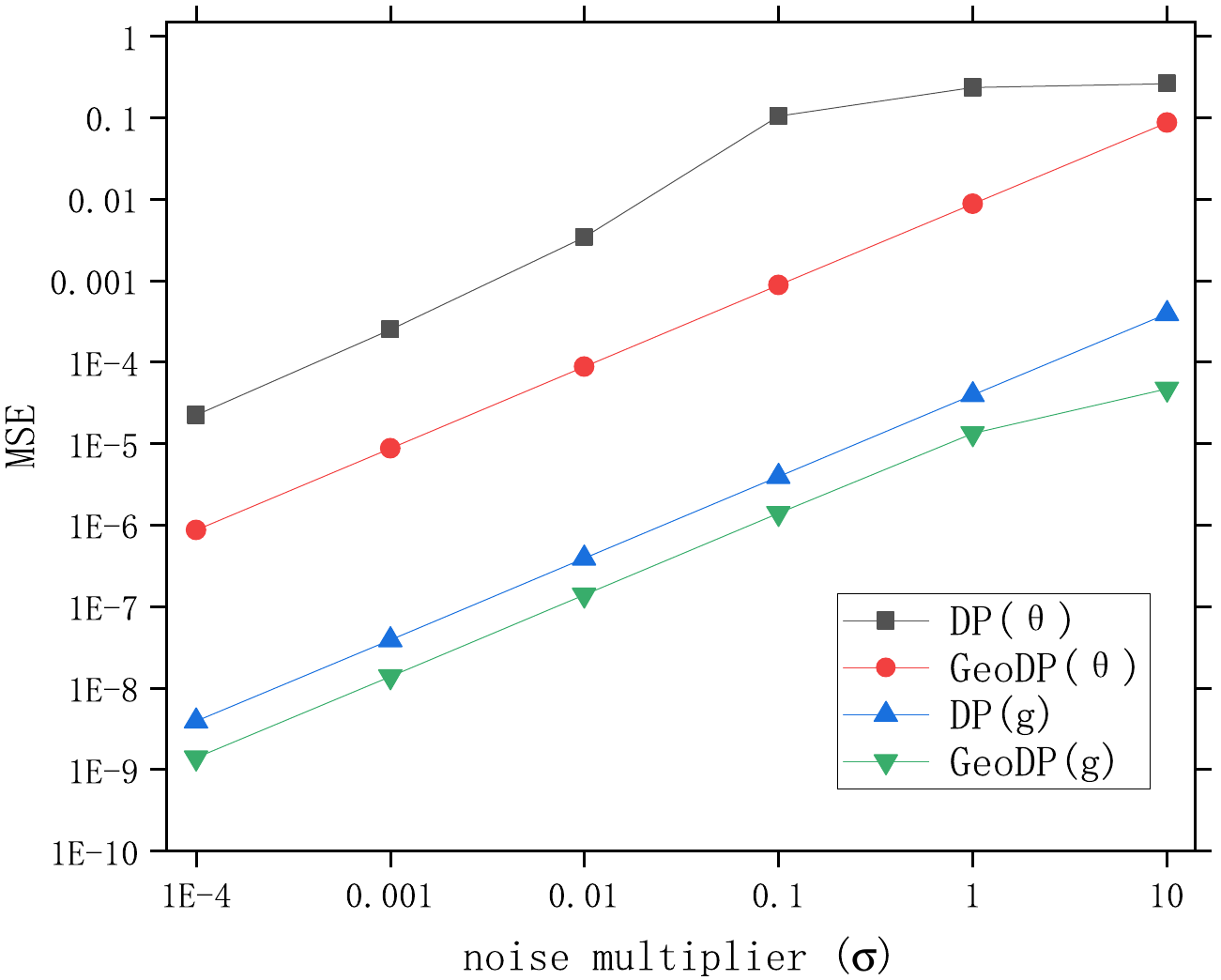}
			\label{subfig:Sigma1}
		}
		\subfigure[$d=5000, B=2048, \beta=0.01$]{
			\includegraphics[width=0.28\linewidth]{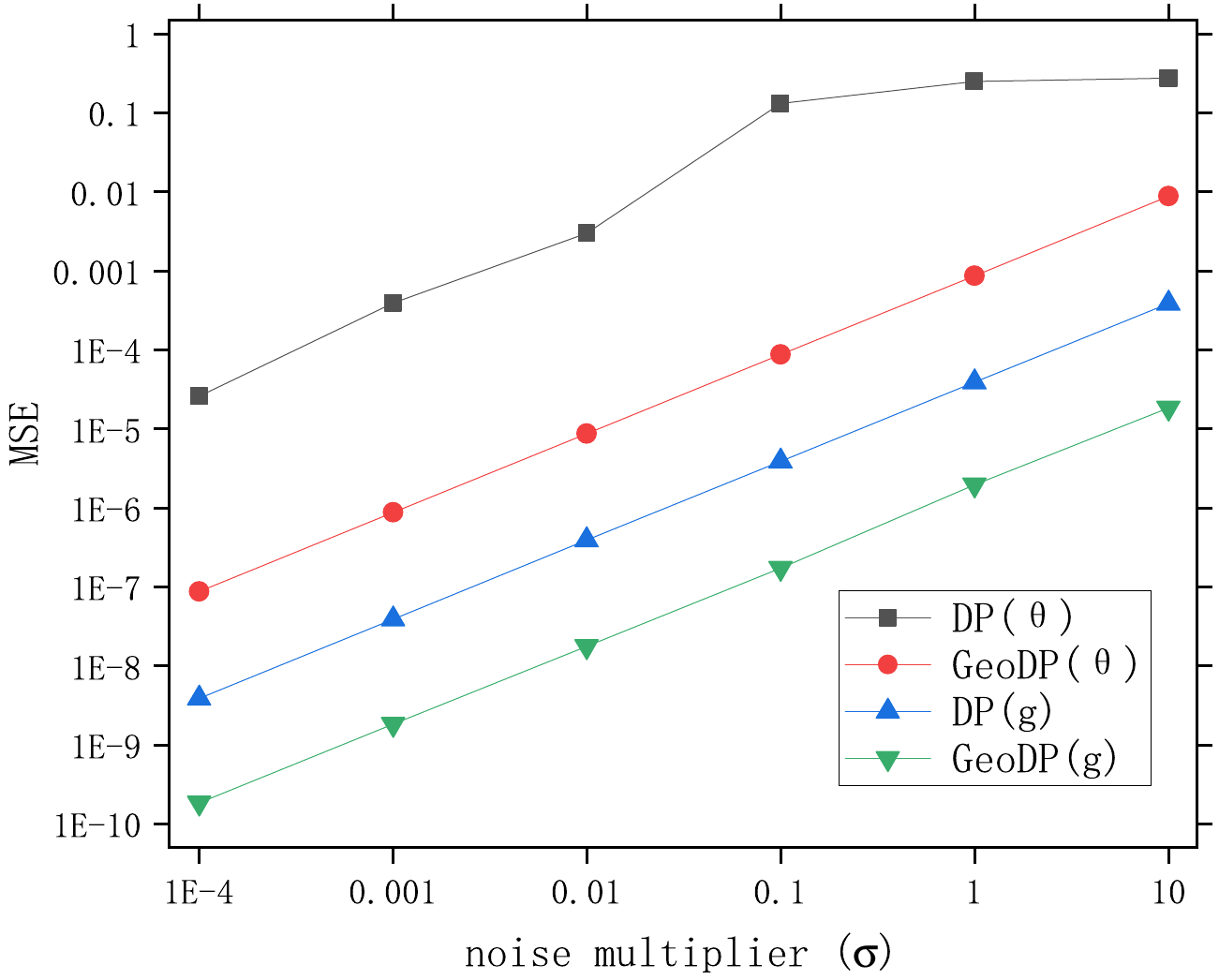}
			\label{subfig:Sigma2}
		}
		\subfigure[$\sigma=8, B=4096, \beta=1$]{
			\includegraphics[width=0.28\linewidth]{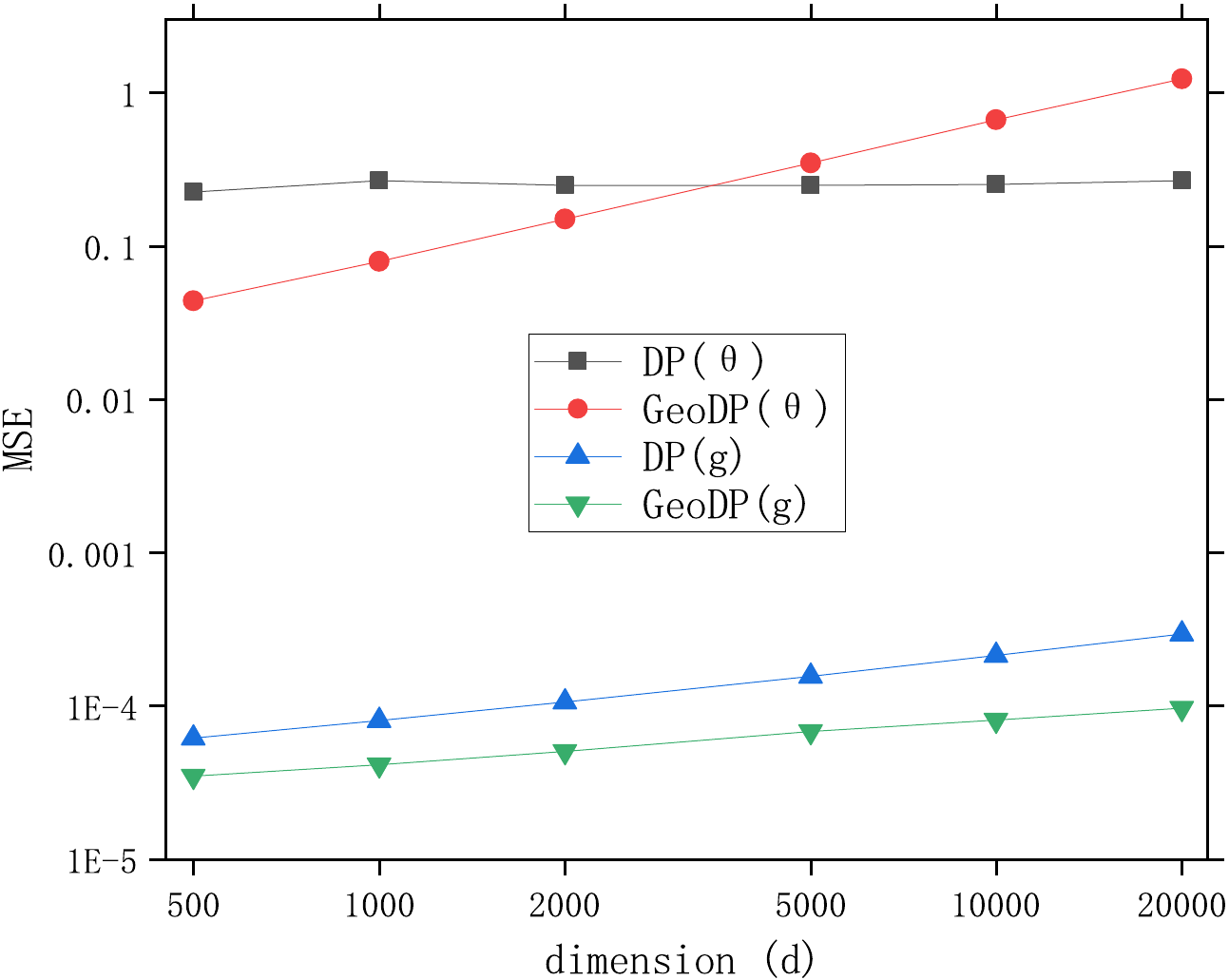}
			\label{subfig:Dimension}
		}
		\subfigure[$\sigma=8, B=4096, \beta=0.1$]{
			\includegraphics[width=0.28\linewidth]{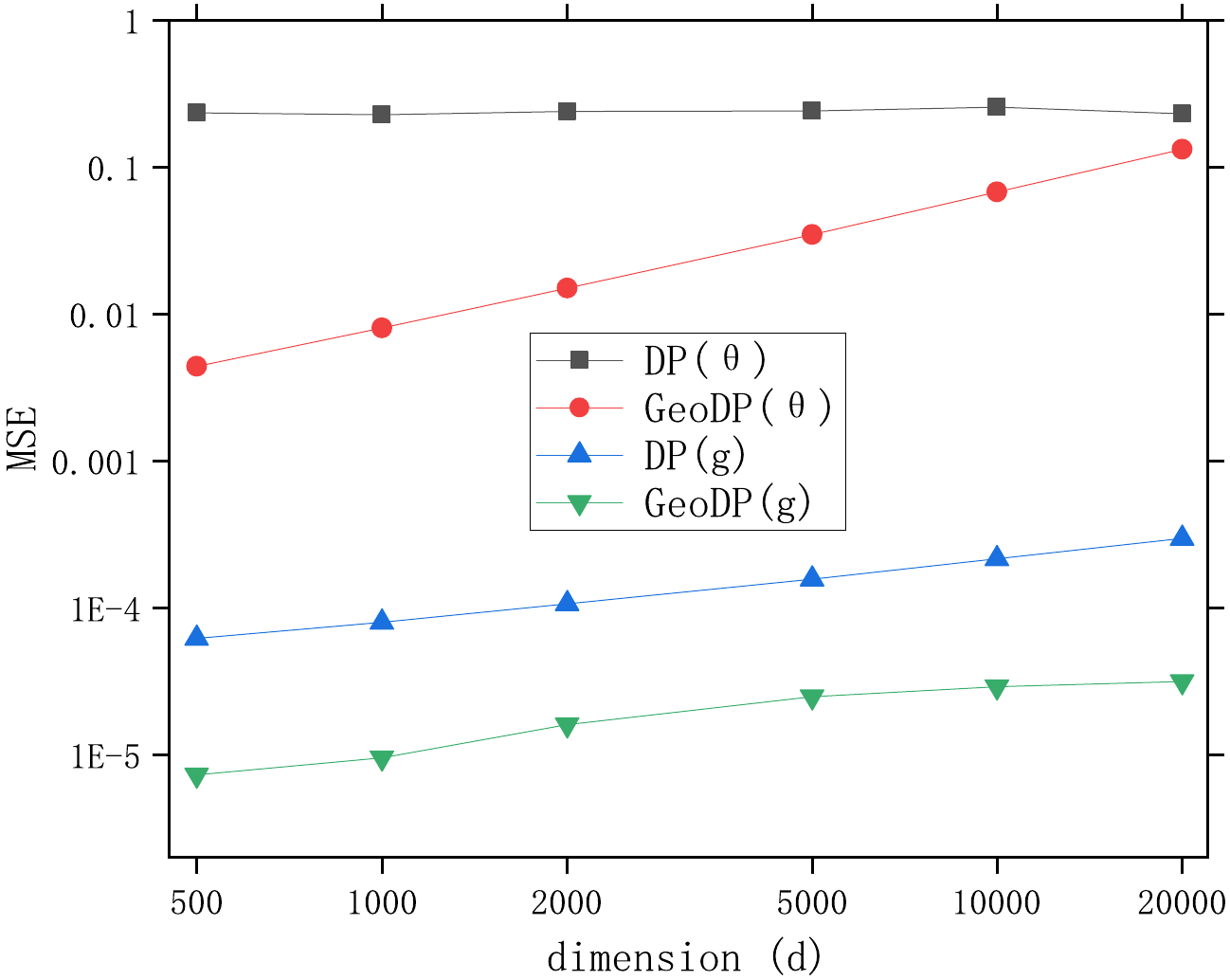}
			\label{subfig:Dimension1}
		}
		\subfigure[$\sigma=8, B=4096, \beta=0.01$]{
			\includegraphics[width=0.28\linewidth]{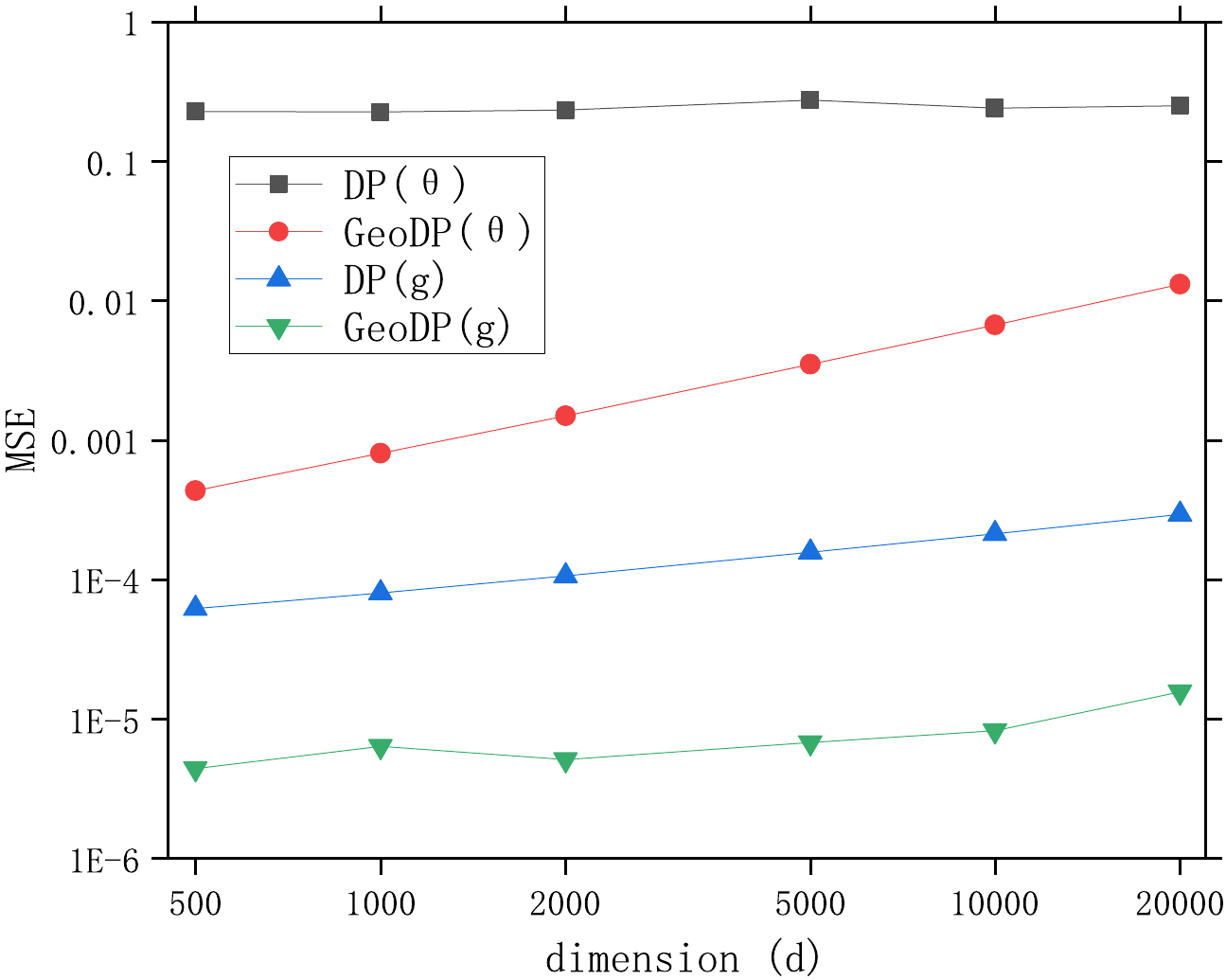}
			\label{subfig:Dimension2}
		}		
		\subfigure[$d=10000,\sigma=8,\beta=1$]{
			\includegraphics[width=0.28\linewidth]{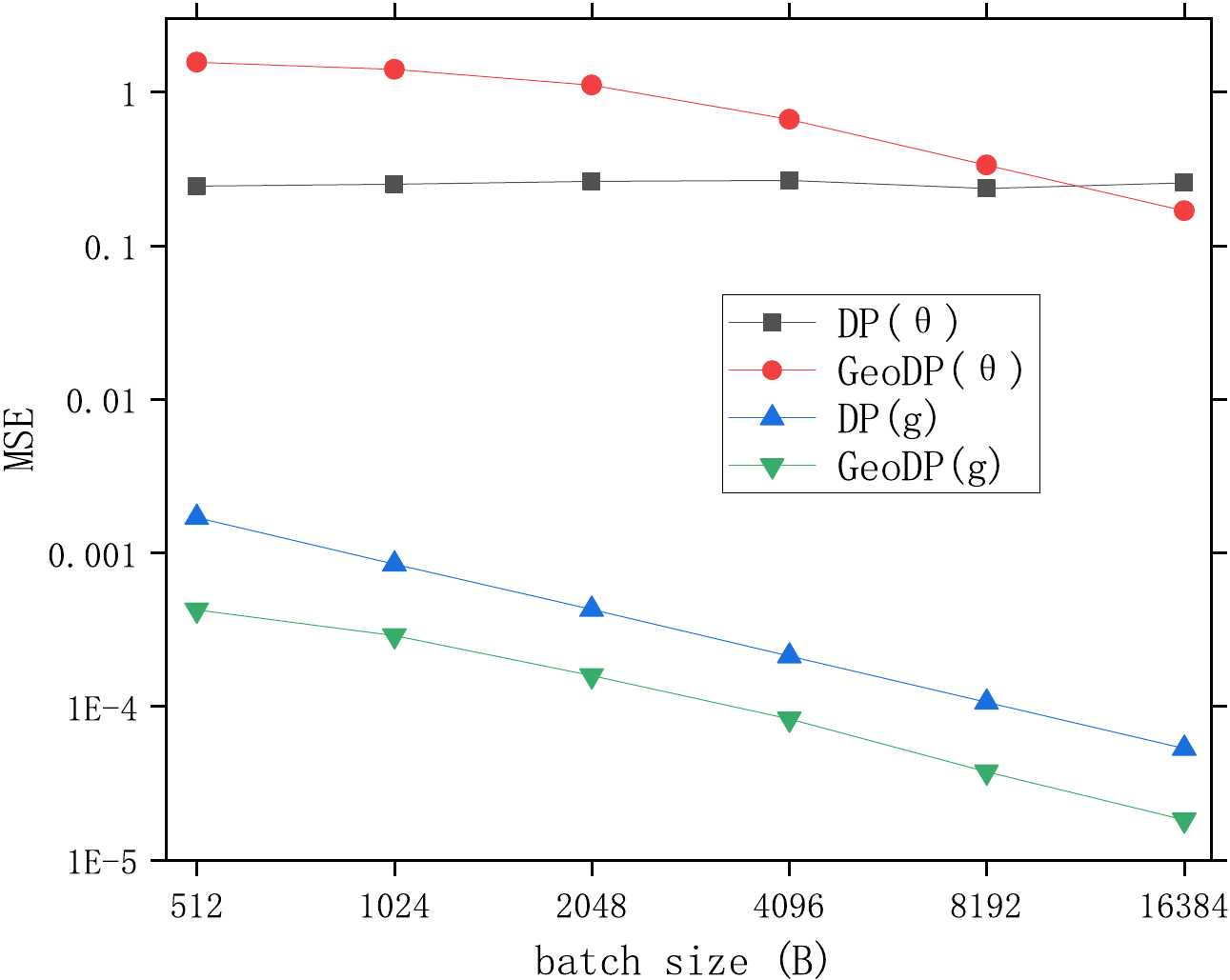}
			\label{subfig:Batch}
		}	
		\subfigure[$d=10000,\sigma=8,\beta=0.1$]{
			\includegraphics[width=0.28\linewidth]{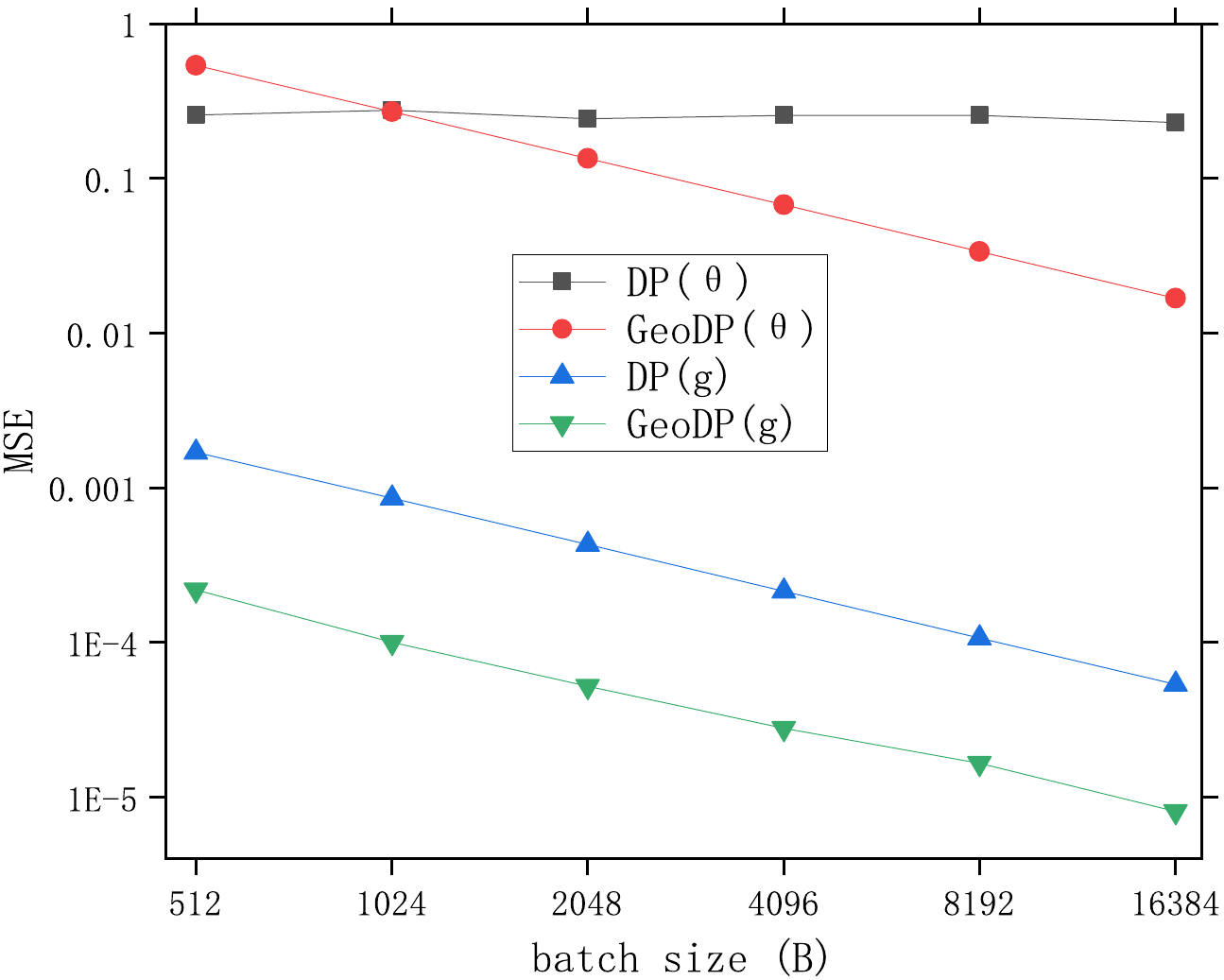}
			\label{subfig:Batch1}
		}	
		\subfigure[$d=10000,\sigma=8,\beta=0.01$]{
			\includegraphics[width=0.28\linewidth]{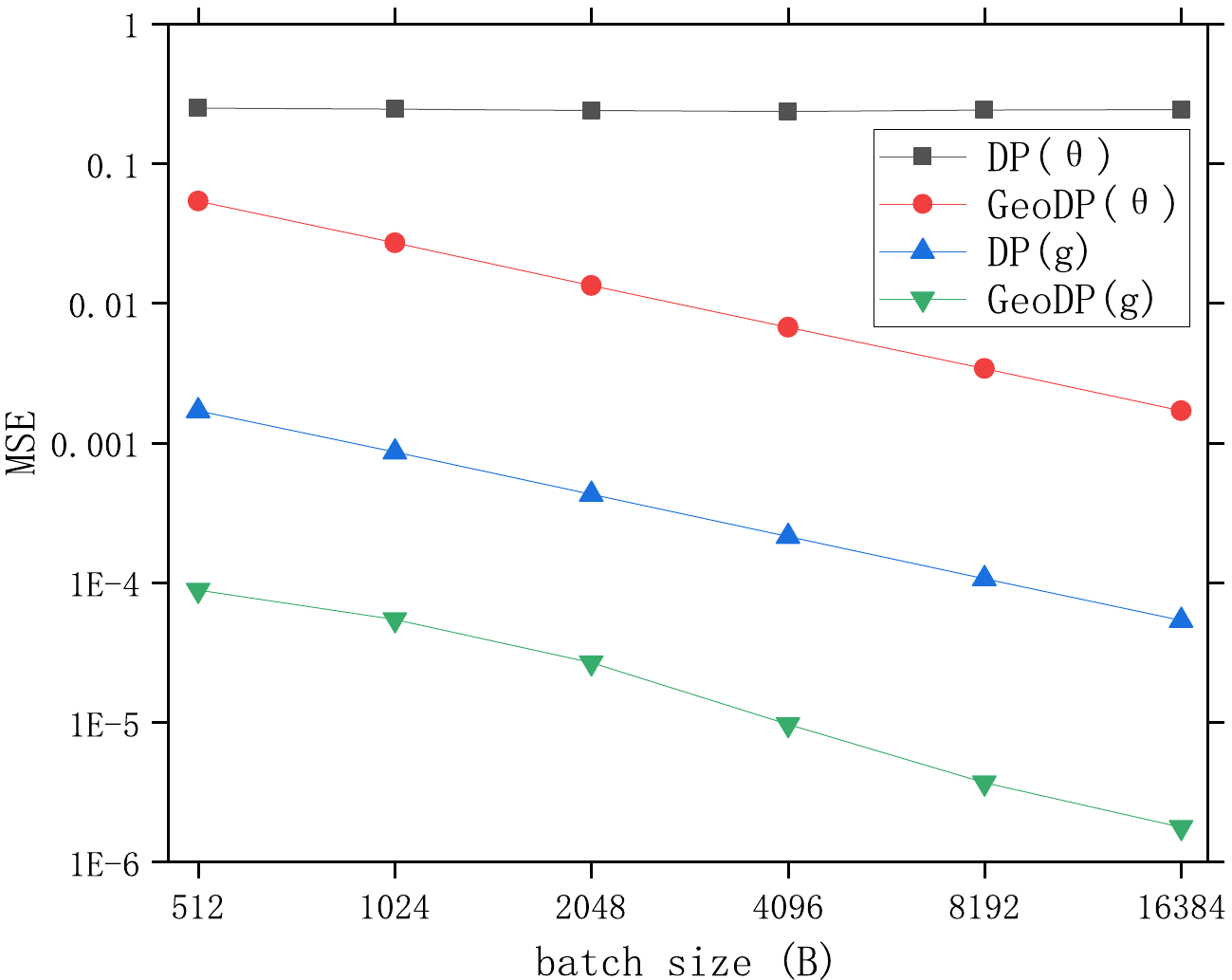}
			\label{subfig:Batch2}
		}			
	\end{center}
	\caption{GeoDP vs. DP on Preserving Gradients under Various Parameters on Synthetic Dataset}
	\label{fig:AngleCompare}
\end{figure*}
\begin{figure*}
	\begin{center}			
		\subfigure[$d=20000,\sigma=8,B=4096$]{
			\includegraphics[width=0.28\linewidth]{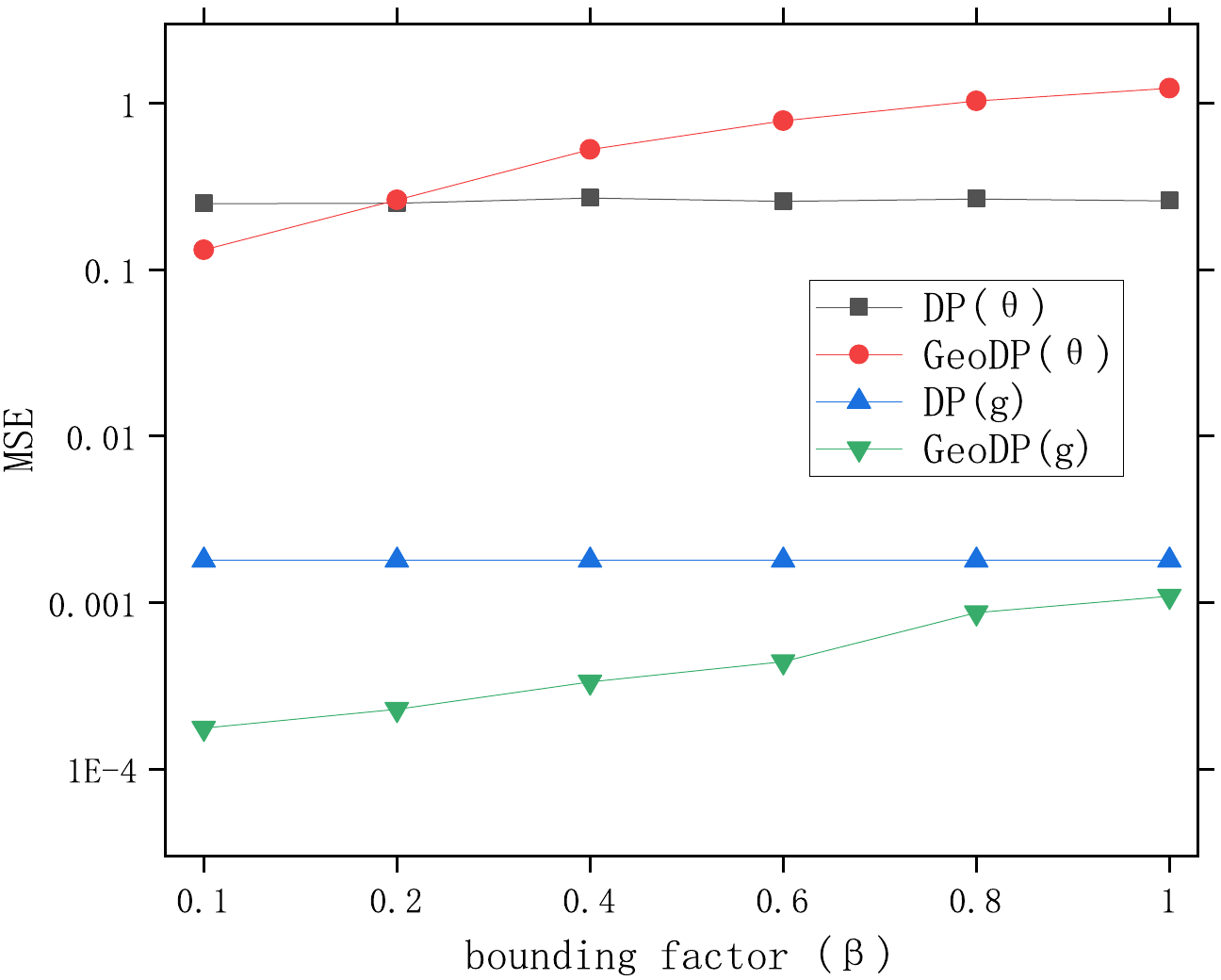}
			\label{subfig:Beta}
		}
		\subfigure[$d=10000,\sigma=8,B=4096$]{
			\includegraphics[width=0.28\linewidth]{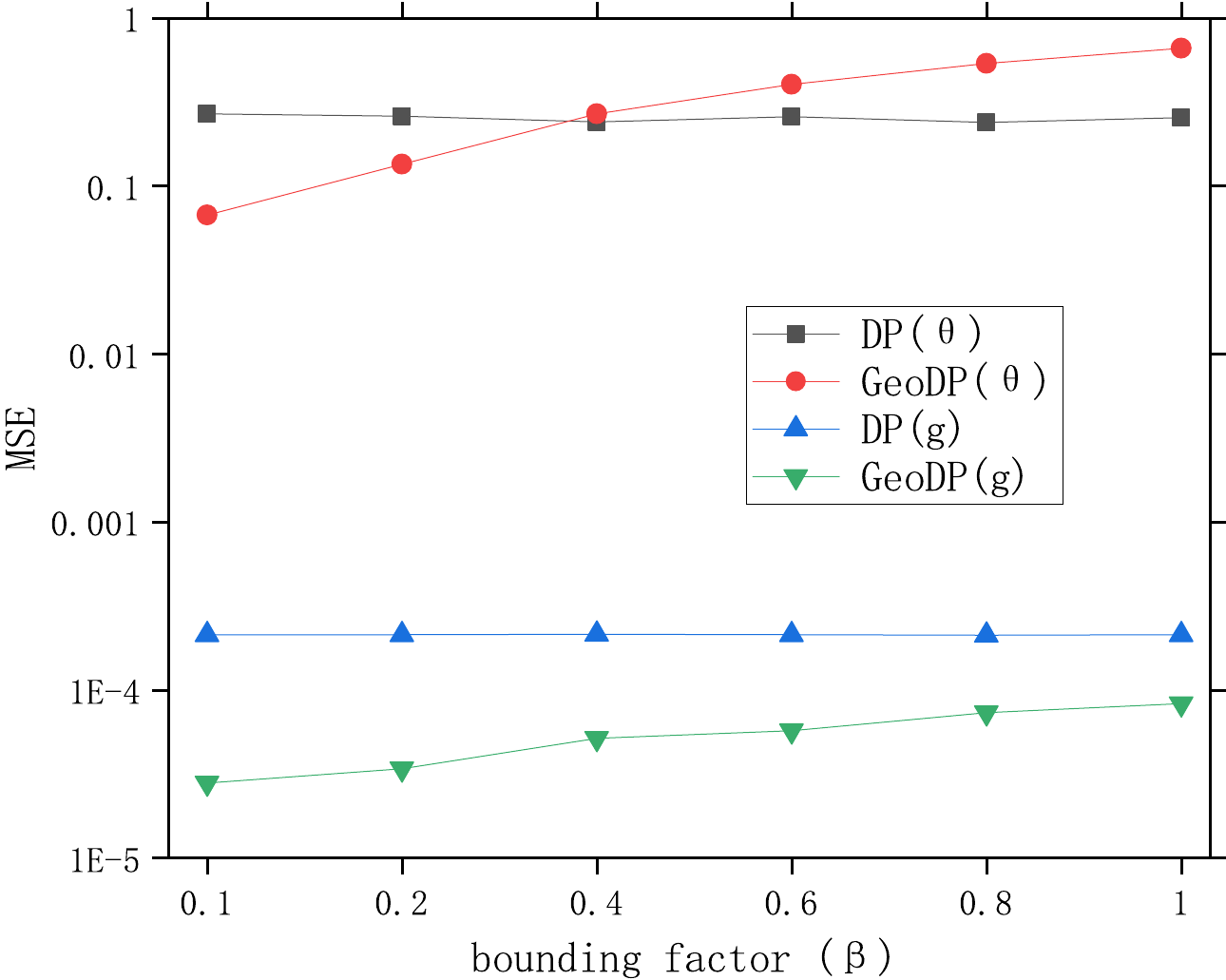}
			\label{subfig:Beta1}		
		}
		\subfigure[$d=5000,\sigma=8,B=4096$]{
			\includegraphics[width=0.28\linewidth]{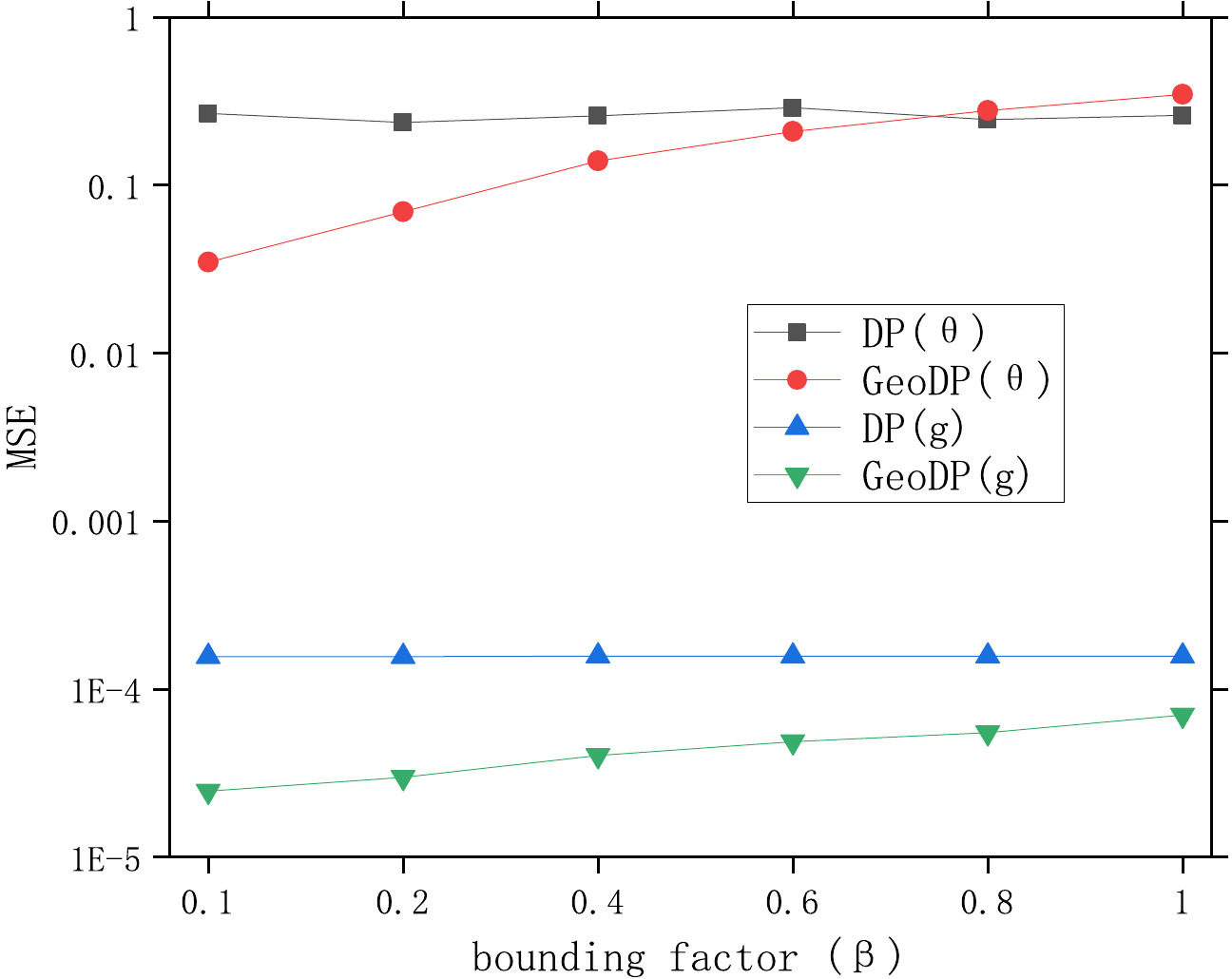}
			\label{subfig:Beta2}
		}
	\end{center}
	\caption{The Effectiveness of Bounding Factor}
	\label{fig:bound}
\end{figure*}
\begin{figure*}  
	\begin{center}
		\subfigure[$d=785, \sigma=1, \beta=1$]{
			\includegraphics[width=0.28\linewidth]{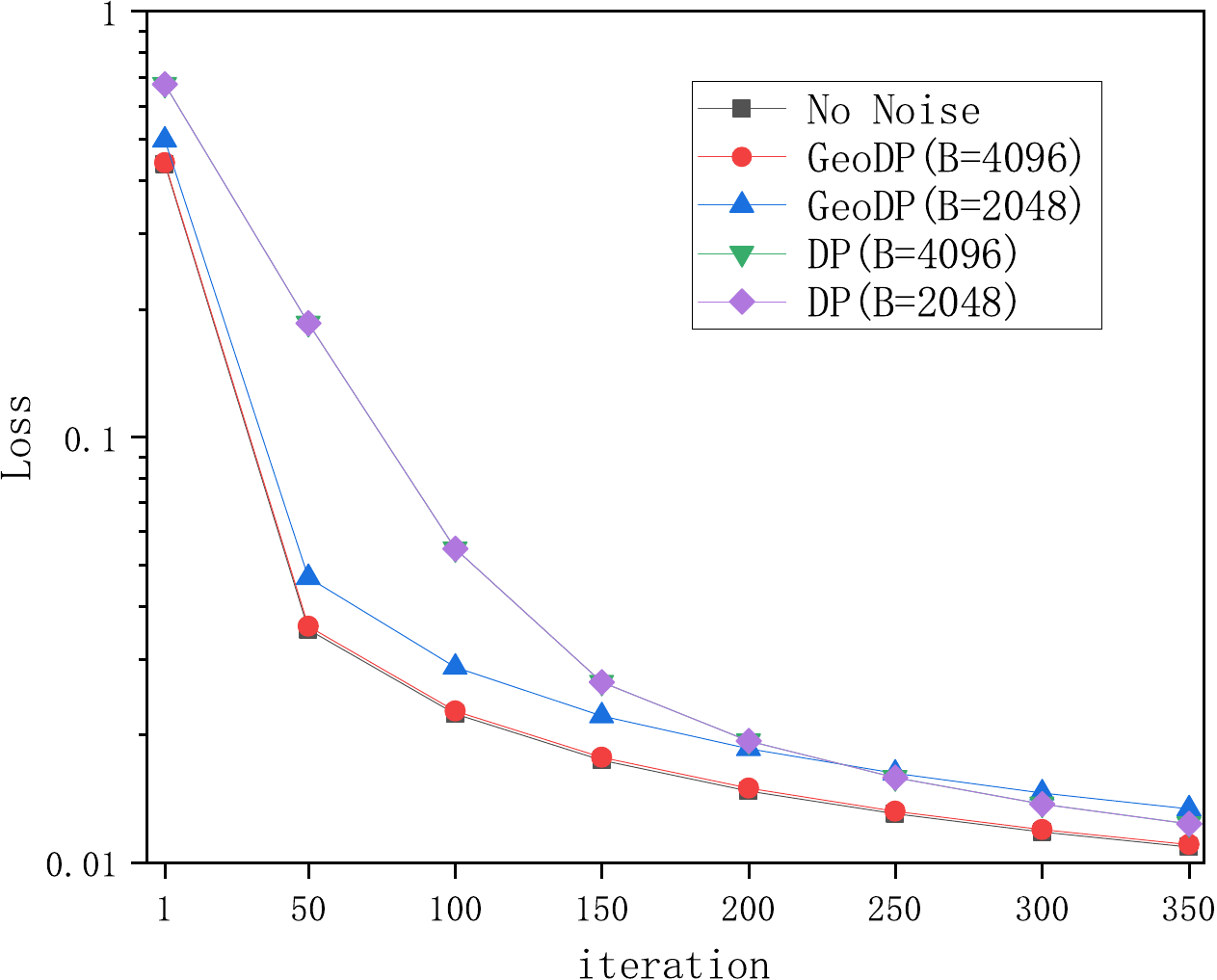}
			\label{subfig:LRSigma1}
		}
		\subfigure[$d=785, \sigma=10, B=2,048$]{
			\includegraphics[width=0.28\linewidth]{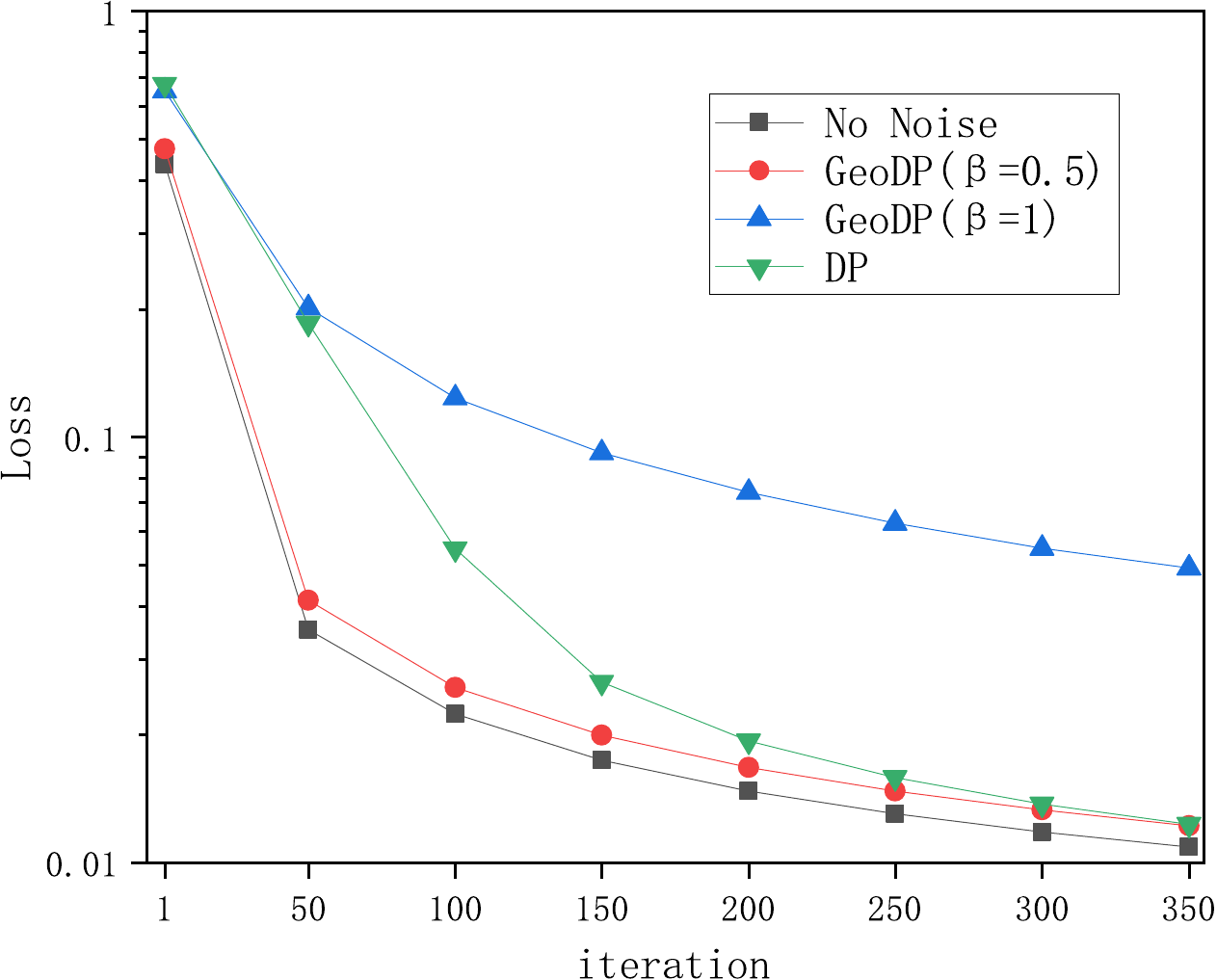}
			\label{subfig:LRSigma10}
		}
		\subfigure[$d=785, \beta=1, B=256$]{
			\includegraphics[width=0.28\linewidth]{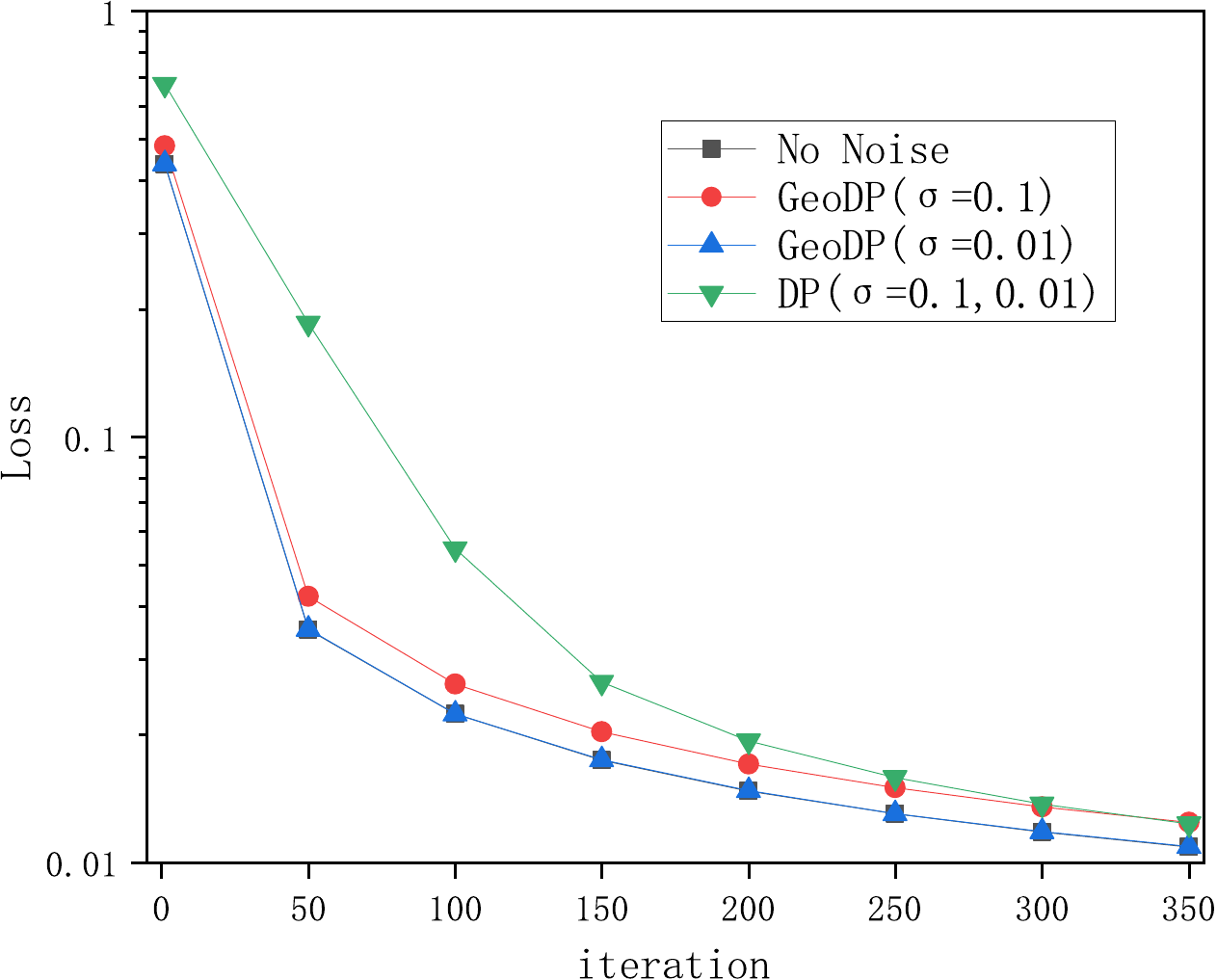}
			\label{subfig:LRB256}
		}
		%		\subfigure[$d=785, \beta=1, B=2,048$]{
			%			\includegraphics[width=0.47\linewidth]{figures/exp/2/LogRegBatch2048.pdf}
			%			\label{subfig:LRB2048}
			%		}
	\end{center}
	\caption{GeoDP versus DP on Logistic Regression under MNIST dataset}
	\label{fig:LR}
\end{figure*}
\subsection{GeoDP vs. DP: Logistic Regression}
\label{subsec:logi}
In the second set of experiments, we verify the effectiveness of GeoDP on Logistic Regression (LR) under MNIST dataset. Figure \ref{fig:LR} plots training losses of $350$ iterations, under \textit{No noise}, \textit{GeoDP} and \textit{DP}. In Figure \ref{subfig:LRSigma1}, with $B=4,096$, GeoDP (the red line) significantly outperforms DP (the green line) and almost has the same performance as noise-free training (black line). The green line overlaps with the purple line because losses of DP-SGD with $B=2,048$ and $B=4,096$ are almost the same. This observation coincides with that from Figure \ref{subfig:Batch}, i.e., the batch size of DP-SGD hardly impacts the noise on the descent trend and thus the model efficiency. In contrast, batch size can successfully reduce the noise of GeoDP (see the gap between the red and blue lines).

In Figure \ref{subfig:LRSigma10}, we test the performance of GeoDP under large noise scale. Initially, GeoDP (blue line) performs worse than DP (green line) with $\beta=1$. When reducing $\beta$ to $0.5$ as suggested in Section \ref{subsec:perang}, the performance of GeoDP surges and leaves DP behind. This observation confirms the superiority of GeoDP over DP even under extreme cases.

In Figure \ref{subfig:LRB256}, we fix the $\beta=1$ and $B=256$ while varying the noise multiplier in $\sigma=\left\{0.01,0.1\right\}$. As we can see, reducing $\sigma$ cannot help DP to perform better (see the green line). This is because DP introduces biased noise to the direction, as confirmed by Lemma \ref{lem:MSE}. Simply reducing the variance of noise cannot counteract this bias. As such, \textbf{DP is sub-optimal even under very small multiplier.} By contrast, GeoDP can achieve significant efficiency improvement with multiplier reduction. When $\sigma=0.01$ (see the blue line), GeoDP almost achieves noise-free model efficiency (the blue line is only slightly above the black line).

%Similar results can be also observed in Figure \ref{subfig:LRB2048}, where GeoDP even achieves noise-free model efficiency with both $\sigma=0.01,0.1$ while DP cannot further achieve better model efficiency with noise multiplier reduced. 
\begin{table}[h]
	\centering	
	\scalebox{0.91}{
		\begin{tabular}{cccc}
			\Xhline{1pt}
			Dataset & Method & $\sigma=10$ & $\sigma=1$  \\
			\Xhline{1pt}
			\multirow{15}*{\makecell{MNIST \\  (noise-free\\ 99.11\%)} }
			& DP ($B=8192$) &  87.93\%& 94.25\% \\
			& DP ($B=16384$)&  88.12\%& 95.52\%  \\
			&DP+IS ($B=16384$)& 88.43\%& 95.63\%\\
			&DP+SUR ($B=16384$)& 88.47\%& 95.75\%\\
			&DP+AUTO-S ($B=16384$)& 88.40\%& 95.71\%\\
			&DP+PSAC ($B=16384$)& 88.48\%& 95.83\%\\
			&DP+SUR+PSAC ($B=16384$)& 89.83\%& 96.91\%\\
			& GeoDP ($B=8192, \beta=0.1$) &  90.31\%& 96.47\% \\	
			& GeoDP ($B=16384, \beta=0.1$)&  93.58\%& 98.04\%  \\
			& GeoDP ($B=8192, \beta=0.5$)& 53.80\%& 60.31\%  \\
			& GeoDP+IS ($B=16384, \beta=0.1$)&  93.60\%& 98.13\%  \\
			& GeoDP+SUR ($B=16384, \beta=0.1$)&  93.68\%& 98.22\%  \\
			& GeoDP+AUTO-S ($B=16384, \beta=0.1$)&  93.64\%& 98.17\%  \\
			& GeoDP+PSAC ($B=16384, \beta=0.1$)&  94.13\%& 98.24\%  \\
			&GeoDP+SUR+PSAC ($B=16384,\beta=0.1$)& 95.27\%& 98.69\%\\
			\Xhline{1pt}		
	\end{tabular}}
	\caption{GeoDP vs. DP on CNN under MNIST Dataset: Test Accuracy}
	\label{tab:accuracy}	
\end{table}
\subsection{GeoDP vs. DP: Deep Learning}
\label{subsec:CNN}
To demonstrate the effectiveness of GeoDP in various learning tasks, we also conduct experiments on MNIST dataset with Convolutional Neural Network (CNN) and Residual Network (ResNet). Due to the extremely large number of parameters, we set the number of training epochs to $20$. While GeoDP pays much attention on the direction, the noisy magnitude is also impacting the overall model efficiency. This is why GeoDP also clips the magnitude before adding noise to it (see Step \ref{item:Per} in Algorithm \ref{algo:GeoDP}). Since the $L_2$-norm of the gradient (i.e., the magnitude) is clipped in existing works~\cite{bu2024automatic,zhang2022understanding}, the same techniques can also be applied to GeoDP. As such, we also demonstrate the generality of GeoDP by integrating it to the state-of-the-art clipping technique AUTO-S~\cite{bu2024automatic}.

Major results are demonstrated in Table \ref{tab:accuracy}. In general, GeoDP outperforms DP under various parameters except for large $\beta$. We can observe that the test accuracy is dramatically reduced (e.g., $96.47\%\to60.31\%$) when $\beta$ increases from $0.1$ to $0.5$. The reason behind is the extremely large sensitivity of GeoDP incurred by high dimensionality ($21,840$ dimensions), as discussed in \ref{subsec:perang}. Overall, we can always find such a $\beta$ ($\beta=0.1$ in this experiment) that GeoDP outperforms DP in any task. Similar results in Table \ref{tab:Resaccuracy} also demonstrates the effectiveness of GeoDP on ResNet under CIFAR-10 dataset. Similar to our observations on LR, GeoDP even better outperforms DP under smaller noise multiplier (e.g., GeoDP can achieve better accuracy than DP even under $\beta=1$.). \textbf{Note that the perturbed direction of GeoDP is unbiased while that of DP is biased, as previously confirmed in Lemma \ref{lem:MSE}}. As such, the optimality of GeoDP over DP under smaller noise multiplier is a reflection of this nature. 

At last, we discuss how to choose $\beta$. In general, $\beta$ is relevant to the model structure, the dataset and the training objective. Compared with CNN under MNIST dataset (Table \ref{tab:accuracy}), ResNet has more complicated structure and CIFAR-10 is more difficult to train (Table \ref{tab:Resaccuracy}). In this case, less $\beta$ should be applied to the latter task for satisfying model efficiency. Besides, $\beta$ can be slightly large if the training objective is not so rigid on model efficiency.
\begin{table}[h]
	\centering	
	\scalebox{0.89}{
		\begin{tabular}{cccc}
			\Xhline{1pt}
			Dataset & Method & $\sigma=0.1$ & $\sigma=0.01$  \\
			\Xhline{1pt}
			\multirow{15}*{\makecell{CIFAR-10 \\  (noise-free\\ 67.43\%)} }
			& DP ($B=8192$) &  59.39\%& 63.27\% \\
			& DP ($B=16384$)&  60.12\%& 63.84\%  \\
			&DP+IS ($B=16384$)& 60.27\%& 64.07\% \\			
			&DP+SUR ($B=16384$)& 61.73\%& 64.83\% \\
			&DP+AUTO-S ($B=16384$)& 60.51\%& 63.91\% \\
			&DP+PSAC ($B=16384$)& 61.30\%& 64.71\% \\
			&DP+SUR+PSAC ($B=16384$)& 62.91\%& 65.60\% \\
			& GeoDP ($B=8192, \beta=1$) &  61.47\%& 65.93\% \\	
			& GeoDP ($B=16384, \beta=1$)&  63.38\%& 66.51\%  \\
			%			& GeoDP ($B=8192, \beta=0.5$)& 53.8\%& 60.3\%  \\
			& GeoDP ($B=16384, \beta=0.1$)&  65.47\%& 67.35\%  \\
			& GeoDP+IS ($B=16384, \beta=0.1$)&  65.51\%& 67.35\%  \\
			& GeoDP+SUR ($B=16384, \beta=0.1$)&  65.53\%& 67.36\%  \\
			& GeoDP+AUTO-S ($B=16384, \beta=0.1$)&  65.58\%& 67.37\%  \\
			& GeoDP+PSAC ($B=16384, \beta=0.1$)&  65.58\%& 67.38\%  \\
			& GeoDP+SUR+PSAC ($B=16384, \beta=0.1$)&  66.03\%& 67.40\%  \\
			\Xhline{1pt}		
	\end{tabular}}
	\caption{GeoDP vs. DP on ResNet under CIFAR-10 Dataset: Test Accuracy}
	\label{tab:Resaccuracy}	
\end{table}
\subsection{GeoDP vs. DP: Time Complexity}
While it is concluded in Section \ref{subsec:GeoDP} that GeoDP and traditional DP have the same time complexity, the practical runtime of GeoDP is likely longer due to the sequential computation involved in coordinate conversions. To compare the runtime of the two algorithms, we conduct experiments on a synthetic dataset. In each experiment, we randomly choose 500 gradients and register the average runtime of GeoDP and DP, respectively, on perturbing these gradients. Specifically, we combine multiple gradients into a single gradient with higher dimensionality (e.g., an 80,000-dimensional gradient is constructed by merging four 20,000-dimensional gradients) to test the limits of both algorithms. Figure \ref{fig:time} illustrates that both batch size and dimensionality have a significant impact on the runtime of both algorithms, with GeoDP being particularly sensitive to these factors. Similar to DP, an increase in either batch size or dimensionality leads to a longer runtime for GeoDP, primarily due to more frequent calculations and increased memory reading/writing. However, the effect of dimensionality on the runtime of GeoDP is particularly pronounced, making it a more dominant factor in extra runtime.

In the low-dimensional case (e.g., $d=1,250$), the majority of the runtime is spent on memory reading and writing, as the calculations themselves are relatively simple. In this scenario, the runtime of GeoDP is only slightly longer than that of DP, and an increase in batch size results in a simultaneous increase in runtime for both algorithms (as seen in the left halves of red and green lines). However, in the high-dimensional case (e.g., $d=320,000$), the sequential computation required for coordinate conversions causes GeoDP to consume considerably more time than DP (as indicated by the right halves of the red and black lines). Despite this, the model accuracy provided by GeoDP offers a significant practical advantage, and the additional runtime can be mitigated by utilizing a more advanced server or implementing a parallel computing strategy. 
\begin{figure}[htbp]
	\centering
	\includegraphics[width=0.7\linewidth]{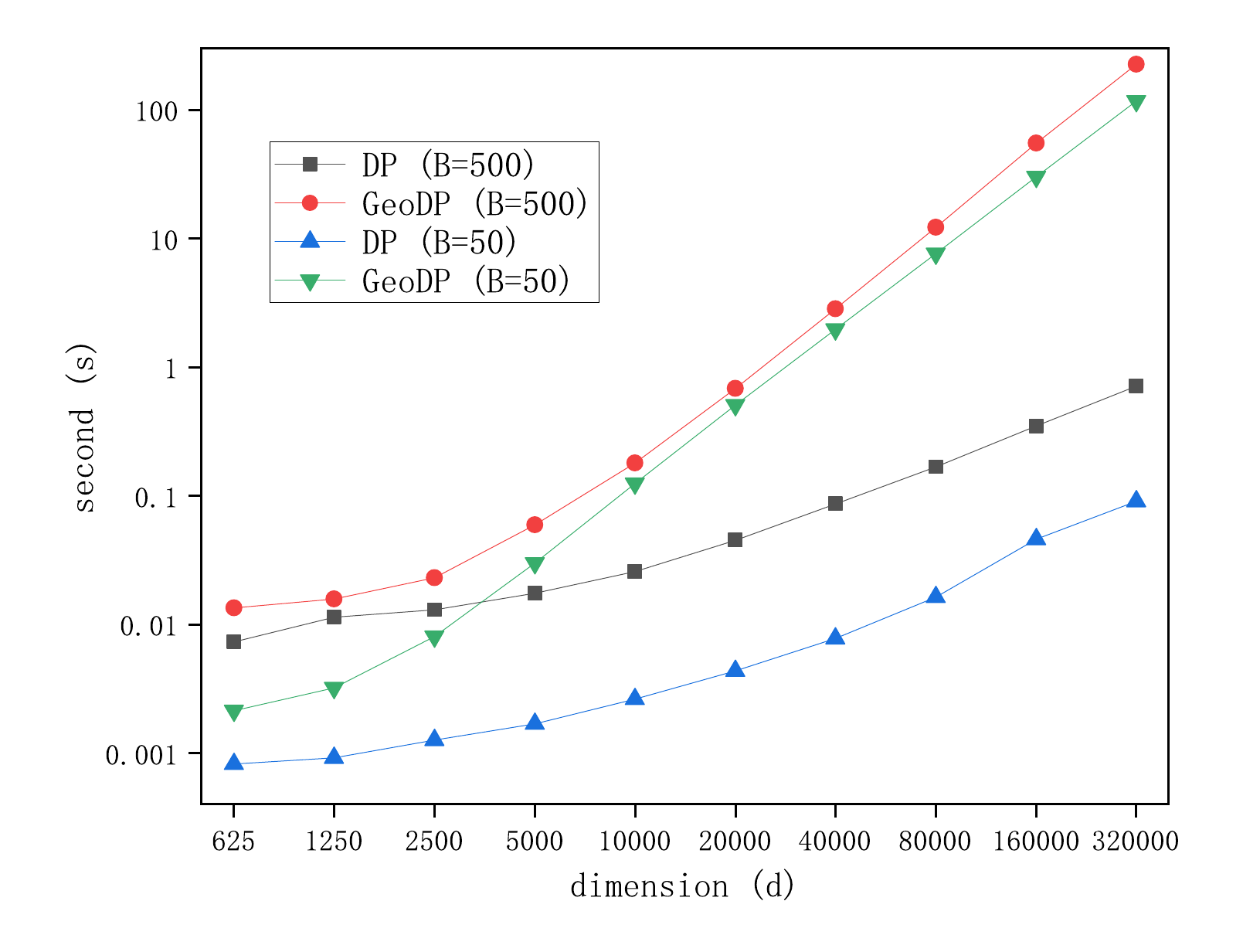}
	\caption{GeoDP vs. DP on Runtime under Various Parameters on the Synthetic Dataset}
	\label{fig:time}
\end{figure}

	\section{Conclusion}
\label{sec:conclusion}
This work optimizes DP-SGD from a new perspective. We first theoretically analyze the impact of DP noise on the training process of SGD, which shows that the perturbation of DP-SGD is actually sub-optimal because it introduces biased noise to the direction. This inspires us to reduce the noise on direction for model efficiency improvement. We then propose our geometric perturbation mechanism GeoDP. Its effectiveness and generality are mutually confirmed by both rigorous proofs and experimental results. As for future work, we plan to study the impact of mainstream training optimizations, such as Adam optimizer~\cite{tang2024dp}, on GeoDP. Besides, we also plan to extend GeoDP to other form of learning, such as federated learning~\cite{gao2023fs}.

	\bibliographystyle{IEEEtran}
	\bibliography{IEEEabrv,ref}
\end{document}